\documentclass[10pt,letterpaper]{article}

\usepackage{amsthm}

\usepackage{caption}
\usepackage{subcaption}
\usepackage[english]{babel}

\usepackage{blindtext}
\usepackage{multicol}

\usepackage[linesnumbered,ruled,vlined]{algorithm2e}

\SetCommentSty{mycommfont}

\SetKwInput{KwInput}{Input}                
\SetKwInput{KwOutput}{Output}              

\newtheorem{theorem}{Theorem}[section]
\newtheorem{proposition}[theorem]{Proposition}
\newtheorem{definition}[theorem]{Definition}
\newtheorem{corollary}[theorem]{Corollary}

\newtheorem{remark}[theorem]{Remark}

\newtheorem{lemma}[theorem]{Lemma}

\newcommand{\iidi}{\stackrel{\mathrm{iid}}{\sim}}

\usepackage{graphicx}
\usepackage{amsmath}
\usepackage{amssymb}
\usepackage{booktabs}
\usepackage{xcolor}
\usepackage{enumerate}
\usepackage{authblk}
\usepackage[pagebackref,breaklinks,colorlinks]{hyperref}

\usepackage[capitalize]{cleveref}
\crefname{section}{Sec.}{Secs.}
\Crefname{section}{Section}{Sections}
\Crefname{table}{Table}{Tables}
\crefname{table}{Tab.}{Tabs.}

\def\cvprPaperID{11738} 
\def\confYear{2023}

\def\Lp#1{\mathrm{L}^{#1}}
\def\dom{\mathrm{dom}}
\def\supp{\mathrm{supp}}
\def\OPT{\mathrm{OPT}}
\def\SOPT{\mathrm{SOPT}}
\def\bbR{\mathbb{R}}
\def\bbS{\mathbb{S}}
\def\argmin{\mathrm{argmin}}

\def\dd{\mathrm{d}}
\def\TV{\mathrm{TV}}

\definecolor{darkblue}{rgb}{0.1,0.1,0.6}
\definecolor{darkgreen}{rgb}{0.1,0.6,0.1}

\newcounter{listmemory}
\newcommand{\tn}[1]{\textnormal{#1}}
\newcommand{\R}{\mathbb{R}}
\usepackage[margin=7em]{geometry}
\begin{document}

\title{Sliced Optimal Partial Transport}
\author[1]{Yikun Bai*}
\author[2]{Bernhard Schmitzer*}
\author[3,4]{Mathew Thorpe}
\author[1]{Soheil Kolouri}
\affil[1]{Department of Computer Science, Vanderbilt University}
\affil[2]{Institute of Computer Science, Göttingen University}
\affil[3]{Department of Mathematics, University of Manchester}
\affil[4]{The Alan Turing Institute}
\affil[1]{\textit{yikun.bai,soheil.kolouri}@vanderbilt.edu}
\affil[2]{\textit
{schmitzer}@cs.uni-goettingen.de}
\affil[3]{\textit
{matthew.thorpe-2}@manchester.ac.uk}
\date{}
\maketitle
\footnotetext{These authors contributed equally to this work.}

\begin{abstract}
Optimal transport (OT) has become exceedingly popular in machine learning, data science, and computer vision. The core assumption in the OT problem is the equal total amount of mass in source and target measures, which limits its application. Optimal Partial Transport (OPT) is a recently proposed solution to this limitation. Similar to the OT problem, the computation of OPT relies on solving a linear programming problem (often in high dimensions), which can become computationally prohibitive. 
In this paper, we propose an efficient algorithm for calculating the OPT problem between two non-negative measures in one dimension. Next, following the idea of sliced OT distances, we utilize slicing to define the sliced OPT distance. Finally, we demonstrate the computational and accuracy benefits of the sliced OPT-based method in various numerical experiments. In particular, we show applications of our proposed Sliced OPT problem in the noisy point cloud registration and color adaptation. Our code is available at \url{https://github.com/yikun-baio/sliced_opt}.
\end{abstract}

\section{Introduction}\label{sec:intro}
The Optimal Transport (OT) problem studies how to find the most cost-efficient way to transport one probability measure to another, and it gives rise to popular probability metrics like the Wasserstein distance. OT has attracted abundant attention in data science, statistics, machine learning, signal processing and computer vision \cite{solomon2014wasserstein,frogner2015learning,montavon2016wasserstein, kolouri2017optimal,arjovsky2017wasserstein,genevay2017gan,liu2019wasserstein,tolstikhin2017wasserstein,courty2014domain,courty2017joint}.

A core assumption in the OT problem is the equal total amount of mass in the source and target measures (e.g., probability measures). Many practical problems, however, deal with comparing non-negative measures with varying total amounts of mass, e.g.,~shape analysis 
\cite{solomon2015convolutional,chizat2018interpolating}, domain adaptation \cite{fatras2021unbalanced}, color transfer \cite{chizat2018scaling}. 
In addition, OT distances are often not robust to outliers and noise, as transporting outliers could be prohibitively expensive and might compromise the distance estimation.
To address these issues, many variants of the OT problem have been recently proposed, for example, the optimal partial transport (OPT) problem \cite{caffarelli2010free,figalli2010optimal,figalli2010new}, the Hellinger--Kantorovich distance \cite{chizat2018interpolating, Liero2018Optimal}, unnormalized optimal transport~\cite{gangbo19},
and Kantorovich--Rubinstein norm \cite{guittet2002extended,lellmann2014imaging}. These variants were subsequently unified under the name ``unbalanced optimal transport'' \cite{chizat2018unbalanced,Liero2018Optimal}. 


The computational complexity of linear programming for balanced and partial OT problems is often a bottleneck for solving large-scale problems. Different approaches have been developed to address this issue. For instance, by entropic regularization, the problem becomes strictly convex and can be solved with the celebrated {Sinkhorn--Knopp} algorithm~\cite{cuturi2013sinkhorn,sinkhorn1964relationship} which has been extended to the unbalanced setting \cite{chizat2018scaling}. This approach can still be computationally expensive for small regularization levels. Other strategies exploit specific properties of ground costs. For example, if the ground cost is determined by the unique path on a tree, the problem can be efficiently solved in the balanced \cite{peyre2019computational,le2019tree} and the unbalanced setting \cite{sato2020fast}. 
In particular, balanced 1-dimensional transport problems with convex ground costs can be solved by the north-west corner rule, which essentially amounts to sorting the support points of the two input measures.

Based on this, another popular method is the sliced OT approach \cite{rabin2011wasserstein,kolouri2015radon,bonneel2015sliced}, which assumes the ground cost is consistent with the Euclidean distance (in 1-dimensional space).
Furthermore, it has been shown \cite{kolouri2015radon,kolouri2016sliced,sato2020fast} that the OT distance in Euclidean space can be approximated by the OT distance in 1-dimensional Euclidean space. Inspired by these works, in this paper, we propose the sliced version of OPT and an efficient computational algorithm for empirical distributions with uniform weights, i.e.,~measures of the form $\sum_{i=1}^n\delta_{x_i}$, where $\delta_{x}$ is the Dirac measure. Our contributions in this paper can be summarized as follows: 

\begin{itemize}
    \item We propose a primal-dual algorithm for 1-dimensional OPT with a quadratic worst-case time complexity and super linear complexity in practice. 
    \item In $d$-dimensional space, we propose the Sliced-OPT (SOPT) distance. Similar to the sliced OT distance, we prove that it satisfies the metric axioms and propose a computational method based on our 1-dimensional OPT problem solver.
    \vspace{+.1in}
    \item We demonstrate an application of SOPT in point cloud registration by proposing a SOPT variant of the \textit{iterative closest point} (ICP) algorithm. Our approach is robust against noise. Also, we apply SOPT to a color adaptation problem. 
\end{itemize}

\section{Related Work}
\noindent\textbf{Linear programming}. In the discrete case, the Kantorovich formulation~\cite{kantorovich1948problem} of OT problem is a (high-dimensional) linear program \cite{karmarkar1984new}. As shown in \cite{caffarelli2010free}, OPT can be formulated as a balanced OT problem by introducing \textit{reservoir} points, thus it could also be solved by linear programming. However, the time complexity is prohibitive for large datasets.

\noindent\textbf{Entropy Regularization}.
Entropic regularization approaches add the transport plan's entropy to the OT objective function and then apply the Sinkhorn-Knopp algorithm \cite{sinkhorn1964relationship,cuturi2013sinkhorn}. 
The algorithm can be extended to the large-scale stochastic \cite{genevay2016stochastic} and unbalanced setting \cite{benamou2015iterative,chizat2018scaling}.
For moderate regularization these algorithms converge fast, however, there is a trade-off between accuracy versus stability and convergence speed for small regularization.

\noindent\textbf{Sliced OT}. 
Sliced OT techniques \cite{rabin2011wasserstein,kolouri2015radon,bonneel2015sliced,liutkus2019sliced,kolouri2016sliced} rely on the closed-form solution for the balanced OT map in 1-dimensional Euclidean settings, i.e., the increasing re-arrangement function given by the north-west corner rule. The main idea behind these methods is to calculate the expected OT distance between the 1-dimensional marginal distributions (i.e., slices) of two $d$-dimensional distributions. The expectation is numerically approximated via a Monte Carlo integration scheme. Other extensions of these distances include the generalized and the max-sliced Wasserstein distances \cite{kolouri2019generalized,deshpande2019max}.   
In the unbalanced setting, and for a particular case of OPT, \cite{Bonneel2019sliced} propose a fast (primal) algorithm, which has quadratic worst-case time complexity, and often linear complexity in practice. In particular, Bonneel et al. \cite{Bonneel2019sliced} assume that all the mass in the source measure must be transported to the target measure, i.e., no mass destruction happens in the source measure.

\noindent\textbf{Other computational methods}. 
When the transportation cost is a metric, network flow methods \cite{guittet2002extended,orlin1988faster} can be applied. 
For metrics on trees, an efficient algorithm based on dynamic programming with time complexity $\mathcal{O}(n\log^2n)$ is proposed in \cite{sato2020fast}. However, in high dimensions existence (and identification) of an appropriate metric tree remains challenging.

\section{Background of Optimal (Partial) Transport}

We first review the preliminary concepts of the OT and the OPT problems. In what follows, given $\Omega\subset\mathbb{R}^d,p\ge 1$, we denote by $\mathcal{P}(\Omega)$ the set of Borel probability measures and by $\mathcal{P}_p(\Omega)$ the set of probability measures with finite $p$'th moment defined on a metric space $(\Omega,d)$. 

\noindent\textbf{Optimal transport}. Given $\mu,\nu\in \mathcal{P}(\Omega)$, and a lower semi-continuous function $c:\Omega^2\to \mathbb{R}_+$, the OT problem between $\mu$ and $\nu$ in the \textit{Kantorovich formulation} \cite{kantorovich1948problem}, is defined as: 
\begin{align}
    \text{OT}(\mu,\nu):=\inf_{\gamma\in\Gamma(\mu,\nu)}\int_{\Omega^2} c(x,y) \, \dd\gamma(x,y), 
    \label{eq: OT distnace}
\end{align}
where $\Gamma(\mu,\nu)$ is the set of all joint probability measures whose marginal are $\mu$ and $\nu$. Mathematically, we denote as $\pi_{1\#}\gamma=\mu,\pi_{2\#}\gamma=\nu$, where $\pi_1,\pi_2$ are canonical projection maps, and for any (measurable) function $f:\Omega^2\to\Omega$, $f_\#\gamma$ is the push-forward measure defined as $f_\#\gamma(A)=\gamma(f^{-1}(A))$ for any Borel set $A\subset \Omega$. 
When $c(x,y)$ is the $p$-th power of a metric, the $p$-th root of the induced optimal value is the \textbf{Wasserstein distance}, a metric in $\mathcal{P}_p(\Omega)$.
\noindent\textbf{Optimal Partial Transport}.
The OPT problem, in addition to mass transportation, allows mass destruction on the source and mass creation on the target.
Here the mass destruction and creation penalty will be linear. Let $\mathcal{M}_+(\Omega)$ denote the set of all positive Radon measures defined on $\Omega$, suppose $\mu,\nu\in\mathcal{M}_+(\Omega)$, and $\lambda_1,\lambda_2\ge 0$, the OPT problem is:
\begin{align}
&\OPT_{\lambda_1,\lambda_2}(\mu,\nu):=\inf_{\substack{\gamma\in\mathcal{M}_+(\Omega^2) \\ \pi_{1\#}\gamma\leq \mu,\pi_{2\#}\gamma\leq \nu} } \int c(x,y) \, \dd\gamma \label{eq: OPT_org} +\lambda_1(\mu(\Omega) -\pi_{1\#}\gamma(\Omega))+\lambda_2(\nu(\Omega)-\pi_{2\#}\gamma(\Omega)) 
\end{align}
where the notation $\pi_{1\#}\gamma \leq \mu$ denotes that for any Borel set $A\subseteq\Omega$, $\pi_{1\#}\gamma(A)\leq \mu(A)$, and we say $\pi_{1\#}\gamma$ is \textit{dominated by} $\mu$, analogously for $\pi_{2\#}\gamma\leq \nu $; the notation $\mu(\Omega)$ denotes the total mass of measure $\mu$.
We denote the set of such $\gamma$ by $\Gamma_{\leq}(\mu,\nu)$.
When the transportation cost $c(x,y)$ is a metric, and $\lambda_1=\lambda_2$, $\OPT(\cdot,\cdot)$ defines a metric on $\mathcal{M}_+(\Omega)$ (see~\cite[Proposition 2.10]{chizat2018unbalanced}, \cite[Proposition 5]{Piccoli2014Generalized}, \cite[Section 2.1]{lee2021generalized} and~\cite[Theorem 4]{chen2017matricial}). For finite $\lambda_1$ and $\lambda_2$  let $\lambda=\frac{\lambda_1+\lambda_2}{2}$ and define:
 \begin{align}
     \OPT_{\lambda}(\mu,\nu):=&\OPT_{\lambda,\lambda}(\mu,\nu)\label{eq: OPT}\\
     =&\OPT_{\lambda_1,\lambda_2}(\mu,\nu)-K_{\lambda_1,\lambda_2}(\mu,\nu)\nonumber 
 \end{align}
 where, 
 \begin{align*}
     K_{\lambda_1,\lambda_2}(\mu,\nu)= \frac{\lambda_1-\lambda_2}{2}\mu(\Omega)+\frac{\lambda_2-\lambda_1}{2}\nu(\Omega).
 \end{align*}
Since for fixed $\mu$ and $\nu$, $K_{\lambda_1,\lambda_2}$ is a constant, and without the loss of generality, in the rest of the paper, we only consider $\OPT_{\lambda}(\mu,\nu)$. The case where $\lambda_i$s are not finite is discussed in Section \ref{sec:algorithm}. 

Various equivalent formulations of OPT \eqref{eq: OPT} have appeared in prior work, e.g., \cite{figalli2010new,figalli2010optimal, Piccoli2014Generalized}, which were later unified as a special case of unbalanced OT \cite{chizat2018unbalanced, Liero2018Optimal}. We provide a short summary of these formulations and their relationship with unbalanced OT in the appendix. 
OPT has several desirable theoretical properties. For instance, by~\cite[Proposition 5]{Piccoli2014Generalized}, minimizing $\gamma$ exist and they are concentrated on $c$-cyclical monotone sets. 
More concretely, we have the following proposition.
\begin{proposition}\label{pro: cyclical monotonicity}
Let $\gamma^*$ be a minimizer in~\eqref{eq: OPT}, then the support of $\gamma^*$ satisfies the $c$-cyclical monotonicity property: for any $n\in\mathbb{N}$, any $\{(x_i,y_i)\}_{i=1}^n\subset \supp(\gamma^*)$ and any permutation $\sigma:[1:n]\to[1:n]$ we have 
$$\sum_{i=1}^n c(x_i,y_i)\leq \sum_{i=1}^n c(x_i,y_{\sigma(i)}).$$
In particular, in one dimension, for $c(x,y)=f(|x-y|)$ where $f:\mathbb{R}\to\mathbb{R}_+$ is a convex increasing function, c-cyclical monotonicity is equivalent to
\begin{align*}
& (x_1,y_1),(x_2,y_2)\in \supp(\gamma^*) \qquad \Rightarrow \qquad [x_1\leq x_2\text{ and }y_1\leq y_2] \text{ or }[x_1\ge x_2\text{ and }y_1\ge y_2].
\end{align*}
\end{proposition}
\begin{proof} 
Let $\hat{\gamma}$ be optimal for the extended balanced problem of appendix section \ref{sec: OPT and OT}, \eqref{eq: opt-ot m}, and let $\gamma$ be the restriction of this measure to $\Omega \times \Omega$.
Since restriction preserves optimality \cite[Theorem 4.6]{Villani2009Optimal}, $\gamma$ must be an optimal plan between $\pi_{1\#}\gamma$ and $\pi_{2\#}\gamma$ with respect to the (non-extended) cost $c$ on $\Omega \times \Omega$. Therefore, it must be supported on a $c$-cyclically monotone set \cite[Theorem 5.10]{Villani2009Optimal}.
In one dimension, for costs of the form $c(x,y)=f(x-y)$ for convex $f$, $c$-cyclical monotonicity reduces to standard monotonicity, see for instance \cite[Theorem 2.9]{Santambrogio-OTAM}.
\end{proof}
To further simplify \eqref{eq: OPT}, we show in the following lemma that the support of the optimal $\gamma$ does not contain pairs of $(x,y)$ whose cost exceeds~$2\lambda$. 
\begin{lemma}\label{lem: truncated cost}
There exists an optimal $\gamma^*$ for ~\eqref{eq: OPT} such that $\gamma^*(S)=0$, where $S=\{(x,y)\in\Omega^2: c(x,y)\ge 2\lambda\}$. 
\end{lemma}
\begin{proof} 
Pick $\gamma$ and define $\gamma'$ as follows: for any Borel $A\subset \Omega^2$, 
$\gamma'(A)=\gamma(A\setminus S)$. 
Let 
\begin{align}
C(\gamma)&:=\int c(x,y) \, \dd\gamma+ \lambda \left[(\mu(\Omega)-\pi_{1\#}\gamma(\Omega))+(\nu(\Omega)-\pi_{2\#}\gamma(\Omega))\right]   \nonumber \\ 
&=\int \left(c(x,y)-2\lambda\right) \, \dd \gamma+\lambda(\mu(\Omega)+\nu(\Omega)), \nonumber
\end{align}
which is the objective function for the OPT problem defined in~\eqref{eq: OPT}, and the second line follows from the fact $\gamma(\Omega^2)=(\pi_1)_\#\gamma(\Omega)=(\pi_2)_\#\gamma(\Omega)$. Then we have \begin{align}
C(\gamma)-C(\gamma')=\int_{S}(c(x,y)-2\lambda) \, \dd \gamma(x,y)\ge 0 \nonumber
\end{align}
That is, for any $\gamma$, we can find a better transportation plan $\gamma'$ such that $\gamma'(S)=0$. 
\end{proof}

\section{Empirical Optimal Partial Transport}

In $\mathbb{R}^d$, suppose $\mu,\nu$ are $n$ and $m$-size empirical distributions, i.e.,~$\mu=\sum_{i=1}^n\delta_{x_i}$ and $\nu=\sum_{j=1}^m\delta_{y_j}$.
The OPT problem \eqref{eq: OPT}, denoted as $\OPT(\{x_i\}_{i=1}^n,\{y_j\}_{j=1}^m)$ can be written as 
\begin{align}
&\OPT(\{x_i\}_{i=1}^n,\{y_j\}_{j=1}^m):=\min_{\gamma\in\Gamma_{\leq}(\mu,\nu)} \sum_{i,j} c(x_i,y_j)\gamma_{ij}+\lambda(n+m-2\sum_{i,j}\gamma_{ij}) \label{eq: OPT empirical}
\end{align}
where
\[ \Gamma_{\leq}(\mu,\nu):= \{\gamma\in \mathbb{R}_+^{n\times m}: \gamma 1_m\leq 1_n, \gamma^T1_n\leq 1_m\},\footnote{Here the rigorous notation should be $\Pi_\leq (1_n,1_m)$. But we abuse the notation for convenience.} \] 
and $1_n$ denotes the $n\times 1$ vector whose entries are $1$ and analogously for $1_m$. 
We show that the optimal plan $\gamma$ for the empirical OPT problem is induced by a 1-1 mapping. A similar result is known for continuous measures $\mu$ and $\nu$, see \cite[Proposition 2.4 and Theorem 2.6]{figalli2010optimal}.
\begin{theorem}\label{Thm: empirical opt Rd}
There exists an optimal plan for $\OPT(\{x_i\}_{i=1}^n,\{y_j\}_{j=1}^m)$, which is induced by a 1-1 mapping, i.e., $\gamma_{ij}\in\{0,1\},\forall i,j$ and in each row and column, at most one entry of $\gamma$ is 1.
\end{theorem}
\begin{proof} 
We start by adapting the extension \eqref{eq: opt-ot m} to the concrete discrete setting between empirical measures.
Let $\hat{\Omega}=[1:m+n]$, and let $\hat{c} : \hat{\Omega} \times \hat{\Omega}$ be given by
$$\hat{c}(i,j)=\begin{cases}
c(x_i,y_j)-2\lambda & \text{if } i\leq n, j \leq m, \\
0 & \text{otherwise.}
\end{cases}$$
Let $\hat{\mu}=\hat{\nu}=\sum_{i=1}^{m+n} \delta_i$. Then solving the (balanced) optimal transport problem on $\hat{\Omega}$ between $\hat{\mu}$ and $\hat{\nu}$ with respect to cost $\hat{c}$ is (as above) clearly equivalent to the OPT problem, and an optimal OPT plan can be obtained by restricting an optimal $\hat{\gamma}$ to the set $[1:n] \times [1:m]$.
Note that here we have simply split the mass on the isolated point $\hat{\infty}$ onto multiple points, where each only carries unit mass.
At the same time, the set $\Gamma(\hat{\mu},\hat{\nu})$ are the doubly stochastic matrices and by the Birkhoff-von-Neumann theorem, its extremal points are permutation matrices. Thus there always exists an optimal $\hat{\gamma}$ that is a permutation matrix, and thus its restriction to $[1:n] \times [1:m]$ will only contain entries $0$ or $1$, with at most one $1$ per row and column.
\end{proof}

Combining this theorem and the cyclic monotonicity of 1D OPT, we can restrict the optimal mapping to strictly increasing maps.
\begin{corollary}\label{Thm: 1d empirical opt}
For $\{x_i\}_{i=1}^n,\{y_j\}_{j=1}^m$ sorted point lists in $\mathbb{R}$, and a cost function $c(x,y)=h(x-y)$ where $h:\mathbb{R}\to\mathbb{R}$ is strictly convex, the empirical OPT problem $\OPT(\{x_i\}_{i=1}^n,\{y_j\}_{j=1}^m)$ can be further simplified to 
\begin{align}
&\OPT(\{x_i\}_{i=1}^n,\{y_j\}_{j=1}^m):=\min_L C(L) \label{eq: OPT empirical solution}
\end{align}
where
\begin{align*}
&C(L):=\sum_{i\in \dom(L)} c(x_i,y_{L[i]})+\lambda\left(n+m-2\left|\dom(L)\right|\right), 
\end{align*}
and $L:[1:n]\to \{-1\}\cup[1:m]$, $\dom(L):=\{i: L[i]\neq -1\}$, $L_{\mid\dom(L)}:[1:n]\hookrightarrow [1:m]$ is a strictly increasing mapping.\footnote{Here, mapping a point to $\{-1\}$ corresponds to destroying the point in the source, while unmatched points in the target are created. Hence, $L$ uniquely represents the partial transport.}
\end{corollary}
For convenience, we call any mapping $L:[1:n]\to \{-1\}\cup [1:m]$ a ``transportation plan'' (or ``plan'' for short). Furthermore, since $L$ can be represented by a vector, we do not distinguish $L(i)$, $L[i]$, and $L_i$. 
Of course, there is a bijection between admissible $L$ and $\gamma$, and we use this equivalence implicitly in the following.

Importantly, the OPT problem is a convex optimization problem and therefore has a dual form which is given by the following proposition.

\begin{proposition}
\label{prop:Dual}
The primal problem~\eqref{eq:  OPT empirical} admits the dual form
\[ \sup_{\substack{\Phi\in\bbR^n,\Psi\in\bbR^m \\ \Phi_i+ \Psi_j \leq c(x_i,y_j) \, \forall i,j}} \sum_{i=1}^n \min\{\Phi_i,\lambda\} + \sum_{j=1}^m \min\{\Psi_j,\lambda\}. \]
Moreover, the following are necessary and sufficient conditions for $\gamma\in\Gamma_{\leq}(\mu,\nu)$, $\Phi\in\bbR^n$ and $\Psi\in\bbR^m$ to be optimal for the primal and dual problems:
\renewcommand{\arraystretch}{1.5}
\begin{center}
\begin{tabular}{l|l}
    \multicolumn{2}{l}{$\Phi_i+\Psi_j = c(x_i,y_j), \forall (x_i,y_j)\in \mathrm{supp}(\gamma)$} \\
    $\Phi_i < \lambda  \Rightarrow [\pi_{1\#} \gamma]_i = 1$ \qquad & \qquad $\Psi_j < \lambda \Rightarrow [\pi_{2\#} \gamma]_j  = 1$ \\
    $\Phi_i = \lambda  \Rightarrow [\pi_{1\#} \gamma]_i  \in [0,1]$ \qquad & \qquad $\Psi_j = \lambda \Rightarrow [\pi_{2\#} \gamma]_j  \in [0,1]$\\
    $\Phi_i > \lambda \Rightarrow [\pi_{1\#} \gamma]_i = 0$ \qquad & \qquad $\Psi_i > \lambda \Rightarrow [\pi_{2\#} \gamma]_j = 0$. 
\end{tabular}
\end{center}
\end{proposition}

\begin{proof} 
For $\lambda_1=\lambda_2$ we can write (2) as
\[ \OPT_\lambda(\mu,\nu) = \mathrm{ET}(\mu,\nu;\mathcal{F},\mathcal{F})=\inf_{\gamma \geq 0} \int_{\Omega^2} c \dd \gamma + \mathcal{F}(\pi_{1\#}\gamma\parallel \mu) + \mathcal{F}(\pi_{2\#}\gamma\parallel\nu)
\]
where $\mathcal{F}(\hat{\mu}\parallel\mu)$ is called Csiszàr $f-$divergence, defined as 
$$\mathcal{F}(\hat{\mu}\parallel\mu) :=\int F(\sigma)d\mu+F'_\infty\gamma^\perp(\Omega)= \begin{cases}
\lambda \cdot (\mu(\Omega)-\hat{\mu}(\Omega)) & \text{if } 0 \leq \hat{\mu} \leq \mu \\
+\infty & \text{otherwise}
\end{cases},$$
with the integrand $F$
$$ F(s) = \begin{cases}
 \lambda(1-s) & \text{if } s\in[0,1] \\ 
 +\infty & \text{else} 
\end{cases},
$$
and $\sigma,\mu^\perp$ is defined by Lebesgue's decomposition theorem  $\mu=\frac{d\hat\mu}{d\mu}+\mu^\perp$; $F'_{\infty}:=\lim_{s\to \infty}\frac{F(s)}{s}$ is called recession constant of $F$ (We refer \cite[section 2.1]{Liero2018Optimal} for more details). 

By~\cite[Theorem 4.11]{Liero2018Optimal} the dual of $\mathrm{ET}$ is
\[ \sup_{\substack{\Phi\in\mathrm{L}^1(\mu),\Psi\in\mathrm{L}^1(\nu)\\ \Phi\oplus\Psi\leq c\\ \Phi,\Psi \text{ lsc and bounded}}} - \int_\Omega F^*(-\Phi) \, \dd \mu - \int_\Omega F^*(-\Psi) \, \dd \nu \]
where $F^*:\mathbb{R}\to(-\infty,\infty]$, called Legendre conjugate function, is defined as 
\[ F^*(r) = \sup_s ( rs - F(s)) = \max\{-\lambda, r\}. \]
By~\cite[Theorem 4.6]{Liero2018Optimal} the optimality conditions are 
\begin{align}
\Phi \oplus \Psi & = c && \gamma\text{-a.e.} \nonumber \\
-\Phi & \in \partial F\left(\frac{\dd \gamma_1}{\dd \mu}\right), \quad \gamma_1=\pi_{1\#}\gamma && \mu\text{-a.e.} \label{eq:PhiOpt} \\
-\Psi & \in \partial F\left(\frac{\dd \gamma_2}{\dd \nu}\right), \quad \gamma_2=\pi_{2\#}\gamma && \nu\text{-a.e.} \label{eq:PsiOpt} \nonumber
\end{align}
We have
\[ \partial F(s) = \left\{ 
\begin{array}{ll} \{-\lambda\} & \text{if } s\in (0,1) \\ (-\infty,-\lambda] & \text{if } s=0 \\ {[}-\lambda,+\infty) & \text{if } s=1. \end{array}
\right. \]
So~\eqref{eq:PhiOpt} can be written
\begin{align*}
\Phi(x) & = \lambda && \text{if } \frac{\dd \gamma_1}{\dd \mu}(x) \in (0,1) \\
\Phi(x) & \in [\lambda,+\infty) && \text{if } \frac{\dd \gamma_1}{\dd \mu}(x) = 0 \\
\Phi(x) & \in (-\infty,\lambda] && \text{if } \frac{\dd \gamma_1}{\dd \mu}(x) = 1.
\end{align*}
Similarly for $\Psi$.
In the discrete case, the dual problem and optimality conditions reduce to the form stated in the proposition.
\end{proof}


\subsection{A Quadratic Time Algorithm}
\label{sec:algorithm}
Our algorithm finds the solutions $\gamma$, $\Phi$, and $\Psi$ to the primal-dual optimality conditions given in Proposition~\ref{prop:Dual}.
We assume that, at iteration $k$, we have solved $\OPT(\{x_i\}_{i=1}^{k-1},\{y_j\}_{j=1}^m)$ in the previous iteration (stored in $L[1:k-1]$) and we now proceed to solve $\OPT(\{x_i\}_{i=1}^{k},\{y_j\}_{j=1}^m)$ in the current iteration.
Recall that we assume that $\{x_i\}_{i=1}^n$ and $\{y_j\}_{j=1}^m$ are sorted and that $L[i]$ is the index determining the transport of mass from $x_i$, i.e.,~if $L[i]\neq -1$ then $x_i$ is associated to $y_{L[i]}$. For simplicity we assume for now that all points in $\{x_i\}_{i=1}^n$ are distinct, and likewise for $\{y_j\}_{j=1}^m$. Duplicate points can be handled properly with some minor additional steps (see appendix).
%
Let $j^*$ be the index of the most attractive $\{y_j\}_{j=1}^m$ for the new point $x_k$ under consideration of the dual variables, i.e., $j^*=\argmin_{j\in [1:m]} c(x_k,y_j) - \Psi_j$ and set $\Phi_k=\min\{c(x_k,y_{j^*})-\Psi_{j^*},\lambda\}$ so that $\Phi_k$ is the largest possible value satisfying the dual constraints, but no greater than $\lambda$ since at this point the dual objective becomes flat. We now distinguish three cases:
\begin{description}
\item[Case 1:] If $\Phi_k=\lambda$, then destroying $x_k$ is the most efficient action, i.e.,~we set $L[k]=-1$ and proceed to $k+1$.
\item[Case 2:] If $\Phi_k<\lambda$ and $y_{j^*}$ is unassigned, then we set $L[k]=j^*$ and we can proceed to $k+1$.
\item[Case 3:] If $\Phi_k<\lambda$ and $y_{j^*}$ is already assigned, we must resolve the conflict between $x_k$ and the element currently assigned to $y_{j*}$. This will be done by a sub-algorithm.
\end{description}
It is easy to see that in the first two cases, if $L$ (or $\gamma$), $\Phi$ and $\Psi$ satisfy the primal-dual conditions of Proposition \ref{prop:Dual} up until $k-1$, they will now satisfy them up until $k$. Let us now study the third case in more detail.

One finds that if $y_{j^*}$ is already assigned, then it must be to $x_{k-1}$ (proof in appendix).
We now increase $\Phi_i$ for $i\in [k-1:k]$ while decreasing $\Psi_{j^*}$ (at the same rate) until one of the following cases occurs:
\begin{description}
\item[Case 3.1:] Either $\Phi_{k-1}$ or $\Phi_k$ reaches $\lambda$. In this case, the corresponding $x$ becomes unassigned and the other becomes (or remains) assigned to $j^*$. The conflict is resolved and we proceed to solve the problem for $k+1$.
\item[Case 3.2:] One reaches the point where $\Phi_k+\Psi_{j^*+1}=c(x_k,y_{j^*+1})$. In this case, $x_k$ becomes assigned to $y_{j^*+1}$, the conflict is resolved and we move to $k+1$.
\item[Case 3.3a:] One reaches the point where $\Phi_{k-1}+\Psi_{j^*-1}=c(x_{k-1},y_{j^*-1})$. If $y_{j^*-1}$ is unassigned, we assign $x_{k-1}$ to $y_{j^*-1}$, $x_{k}$ to $y_{j^*}$. The conflict is resolved and we move on.
\item[Case 3.3b:] In the remaining case where $y_{j^*-1}$ is already assigned, we will show that it must be to $x_{k-2}$. This means that the set of points involved in the conflict increases and we must perform a slight generalization of the above case 3 iteration until eventually one of the other cases occurs.
\end{description}
At each iteration we consider a contiguous set of $\{x_i\}_{i=i_{\min}}^{k-1}$ assigned monotonously to contiguous $\{y_j\}_{j=j_{\min}}^{j^*}$, where $i_{\min}$ is initially set $i_{\min}=k-1$ and $j_{\min}$ is the index of $y_j$ that has been assigned to $x_{i_{\min}}$ (i.e., $j_{\min}=j^*$), and the additional point $x_k$. We increase $\Phi_i$ for $i \in [i_{\min}:k]$ and decrease $\Psi_j$ for $j \in [j_{\min}:j^\ast]$, until either $\Phi_{k'}$ becomes equal to $\lambda$ for some $k' \in [i_{\min} : k+1]$ (Case 3.1), $\Phi_k+\Psi_{j^*+1}=c(x_k,y_{j^*+1})$ (Case 3.2) or $\Phi_{i_{\min}}+\Psi_{j_{\min}-1}=c(x_{i_{\min}},y_{j_{\min}-1})$ (Case 3.3). In Cases 3.1 and 3.2 the conflict can be resolved in an obvious way. The same holds for Case 3.3, when $y_{j_{\min}-1}$ is unassigned (Case 3.3a). Otherwise (Case 3.3b), one adds one more assigned pair to the chain and restarts the loop with the pair $(x_{i_{\min}-1},y_{j_{\min}-1})$ added to the chain.
Of course, trivial adaptations have to be made to account for boundary effects, e.g.,~Case 3.2 cannot occur if $j^*=m$ etcetera. A pseudocode description of the method is given by Algorithms \ref{alg: 1d opt v1} and \ref{alg: sub opt}, and a visual illustration of these algorithms is provided in Figure \ref{fig: algorithm}. Note that for the sake of legibility, we make some simplifications, e.g.,~boundary checks are ignored (see above), and we do not specify how to keep track of whether $j^*$ is assigned or not. Also, updating the dual variables at each iteration of the sub-routine yields a cubic worst-case time complexity, but is easier to understand. A complete version of the algorithm with all checks, appropriate data structures, and quadratic complexity is given in the appendix. There we also prove the following claim:
\begin{theorem}
    Algorithm \ref{alg: 1d opt v1} is correct, i.e.,~it is well-defined and returns optimal primal and dual solutions $L$ and $(\Phi,\Psi)$ to the 1-dimensional OPT problem for sorted $\{x_i\}_{i=1}^n$, $\{y_j\}_{j=1}^m$. A slight adaptation of the algorithm (given in the appendix) has a worst-case time complexity of $O(n\max\{m,n\})$.
\end{theorem}
The algorithm can be adjusted to the case where $\lambda_1=\infty$ by setting all instances of $\lambda$ in Algorithms  \ref{alg: 1d opt v1} and \ref{alg: sub opt} to $+\infty$ except for the initialization of $\Psi$. This means that cases 1 and 3.1 never occur. The algorithm then reduces to a primal-dual version of that given in \cite{Bonneel2019sliced}. If $\lambda_2=\infty$, then one simply flips the two marginals and proceeds as for $\lambda_1=\infty$.

\begin{algorithm}\caption{opt-1d}
\label{alg: 1d opt v1}
\KwInput{$\{x_i\}_{i=1}^n,\{y_j\}_{j=1}^m,\lambda$}
\KwOutput{$L$, $\Psi$, $\Phi$}
Initialize $\Phi_i\gets-\infty$ for $i\in [1:n]$, $\Psi_j\gets\lambda$ for $j\in [1:m]$ and $L_{i} \gets -1$ for $i\in [1:n]$\\
\For{$k=1,2,\ldots n$}{
$j^*\gets\operatorname{argmin}_{j\in[1:m]} c(x_k,y_j) - \Psi_j$\\
$\Phi_k\gets \min\{c(x_k,y_{j^*}) - \Psi_{j^*},\lambda\}$\\
\If{$\Phi_k=\lambda$}
{{\bf [Case 1]} No update on $L$}
{
\ElseIf{$j^*$ unassigned}
{{\bf [Case 2]} $L_k \gets j^*$ 
}
\Else 
{{\bf [Case 3]} Run Algorithm~\ref{alg: sub opt}. }
}
}
\end{algorithm}
\begin{algorithm}
\caption{sub-opt}\label{alg: sub opt}
{
\KwInput{($\{x_i\}_{i=1}^n,\{y_j\}_{j=1}^m$, $k$, $j^*$, $L$, $\Phi$, $\Psi$)}
\KwOutput{(Updated $L$, $\Phi$, $\Psi$, optimal for $\OPT(\{x_i\}_{i=1}^k,\{y_j\}_{j=1}^m)$)}
Initialize $i_{\min} \gets k-1$, $j_{\min}\gets j^*$.\\
\While{\text{true}}
{
$i_{\Delta} \gets\operatorname{argmin}_{i\in[i_{\min}:k]} (\lambda - \Phi_i)$\\
$\lambda_{\Delta} \gets \lambda - \Phi_{i_\Delta}$\\
$\alpha \gets c(x_k,y_{j^*+1})-\Phi_k-\Psi_{j^*+1}$ \\
$\beta \gets c(x_{i_{\min}},y_{j_{\min}-1})-\Phi_{i_{\min}}-\Psi_{j_{\min}-1}$\\
\If{$\lambda_{\Delta}\leq \min\{\alpha,\beta\}$}
{[\textbf{Case 3.1}]\\
$\Phi_i \gets \Phi_i + \lambda_{\Delta}$ for $i \in [i_{\min}:k]$\\
$\Psi_j \gets \Psi_j - \lambda_{\Delta}$ for $j \in [j_{\min}:j^*]$\\
$L_{i_\Delta} \gets -1$, $L_{k} \gets j^*$ \\
\For{$i \in [i_\Delta+1:k-1]$}{$L_i \gets L_i-1$}
\textbf{return} \\
}
\ElseIf{$\alpha\leq \min\{\lambda_\Delta,\beta\}$}
{[\textbf{Case 3.2}]\\
$\Phi_i \gets \Phi_i + \alpha$ for $i \in [i_{\min}:k]$\\
$\Psi_j \gets \Psi_j - \alpha$ for $j \in [j_{\min}:j^*]$\\
$L_k \gets j^*+1$ \\
\textbf{return}
}
\Else{
$\Phi_i \gets \Phi_i + \beta$ for $i \in [i_{\min}:k]$\\
$\Psi_j \gets \Psi_j - \beta$ for $j \in [j_{\min}:j^*]$\\
\If{$j_{min}-1$ unassigned}{
[\textbf{Case 3.3a}]\\
$L_{i_{\min}} \gets j_{\min}-1$, $L_{k} \gets j^*$ \\
\For{$i \in [i_{\min}+1:k-1]$}{$L_i \gets L_i-1$}
\textbf{return}\\
}
\Else{
[\textbf{Case 3.3b}]\\
$i_{\min} \gets i_{\min}-1$, $j_{\min} \gets j_{\min}-1$
}
}
}
}
\end{algorithm}

\begin{figure*}[t!]
    \centering
    \includegraphics[width=\textwidth]{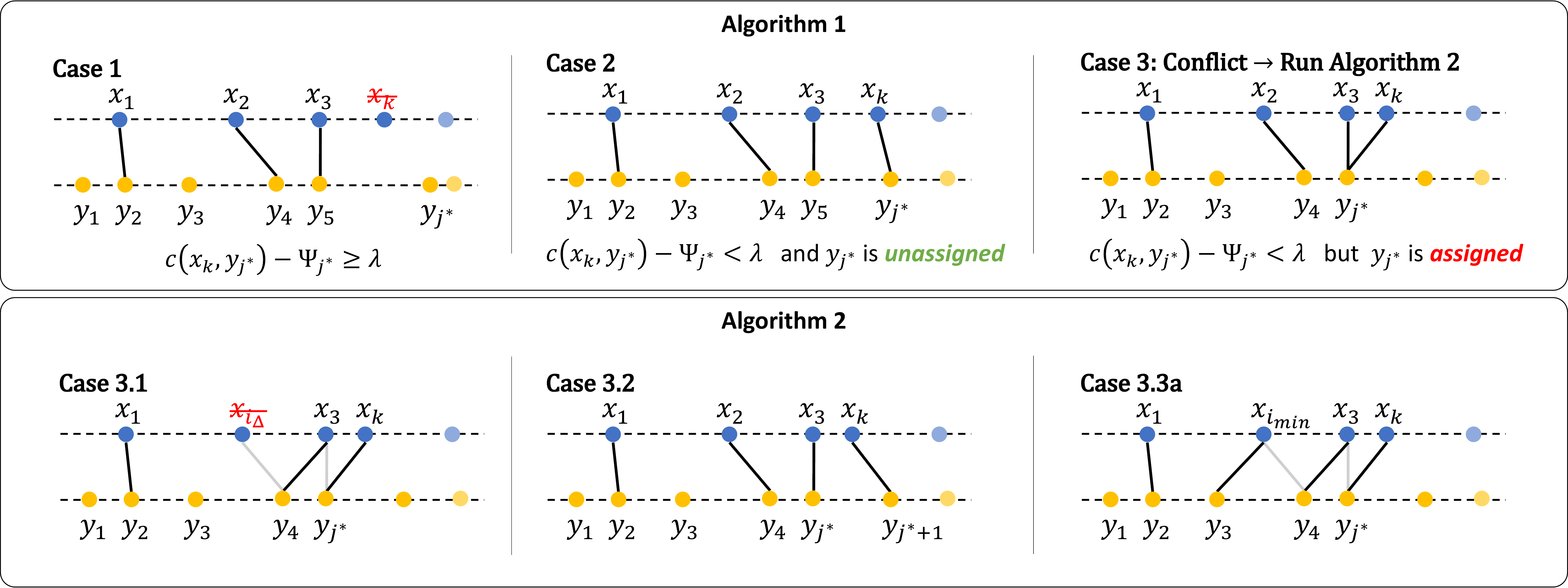}
    \caption{Depiction of Algorithms \ref{alg: 1d opt v1} and \ref{alg: sub opt} for solving the optimal partial transport in one dimension. }
    \label{fig: algorithm}
\end{figure*}

\subsection{Runtime Analysis} 

We test the wall clock time by sampling the point lists from uniform distributions and Gaussian mixtures. In particular, we set $\{x_i\}_{i=1}^n\iidi\text{ Unif}[-20,20]$, $\{y_j\}_{j=1}^m\iidi\text{ Unif}[-40,40]$, $\lambda\in\{20, 100\}$ and $\{x_i\}_{i=1}^n\iidi \frac{1}{5}\sum_{k=1}^{5}\mathcal{N}(-4+2k,1),\{y_j\}_{j=1}^m\iidi \frac{1}{6}\sum_{k=1}^{6}\mathcal{N}(-5+2k,1)$, $\lambda\in\{2,10\}$, $n \in \{500, 1000, 1500,\ldots 20000\}$, and $m=n+1000$. We compare our algorithm with POT (Algorithm 1 in \cite{Bonneel2019sliced}),
an unbalanced Sinkhorn algorithm \cite{chizat2018scaling} (we set the entropic regularization parameter to $\lambda/40$),
and linear programming in \textit{python OT} \cite{flamary2021pot}. Our algorithm, POT, and Sinkhorn algorithm are accelerated by numba (\url{https://numba.pydata.org}) and linear programming is written in C++.
Note that POT \cite{Bonneel2019sliced} and the unbalanced Sinkhorn minimize a different model. In addition, for the latter the performance depends strongly on the strength of regularization and we found that it was not competitive in the regime of low blur.
For each $(n,m)$, we repeat the computation 10 times and compute the average. For our method and POT, the time of sorting is included, and for the linear programming and Sinkhorn algorithms, the time of computing the cost matrix is included. We also visualize our algorithm's solutions for $\{x_i\}_{i=1}^n\iidi \text{Unif}(-20,20), \{y_j\}_{j=1}^m\iidi\text{Unif}(-40,40), n=8, m=16$ and $\lambda\in\{1,10,100,1000\}$ (see figure \ref{fig: opt solution}). One can observe that as $\lambda$ increases larger transportations are permitted.  
The data type is 64-bit float number and the experiments are conducted on a Linux computer with AMD EPYC 7702P CPU with 64 cores and 256GB DDR4 RAM.
\begin{figure}
\centering
\begin{subfigure}{0.49\textwidth}
    \includegraphics[width=1\textwidth]{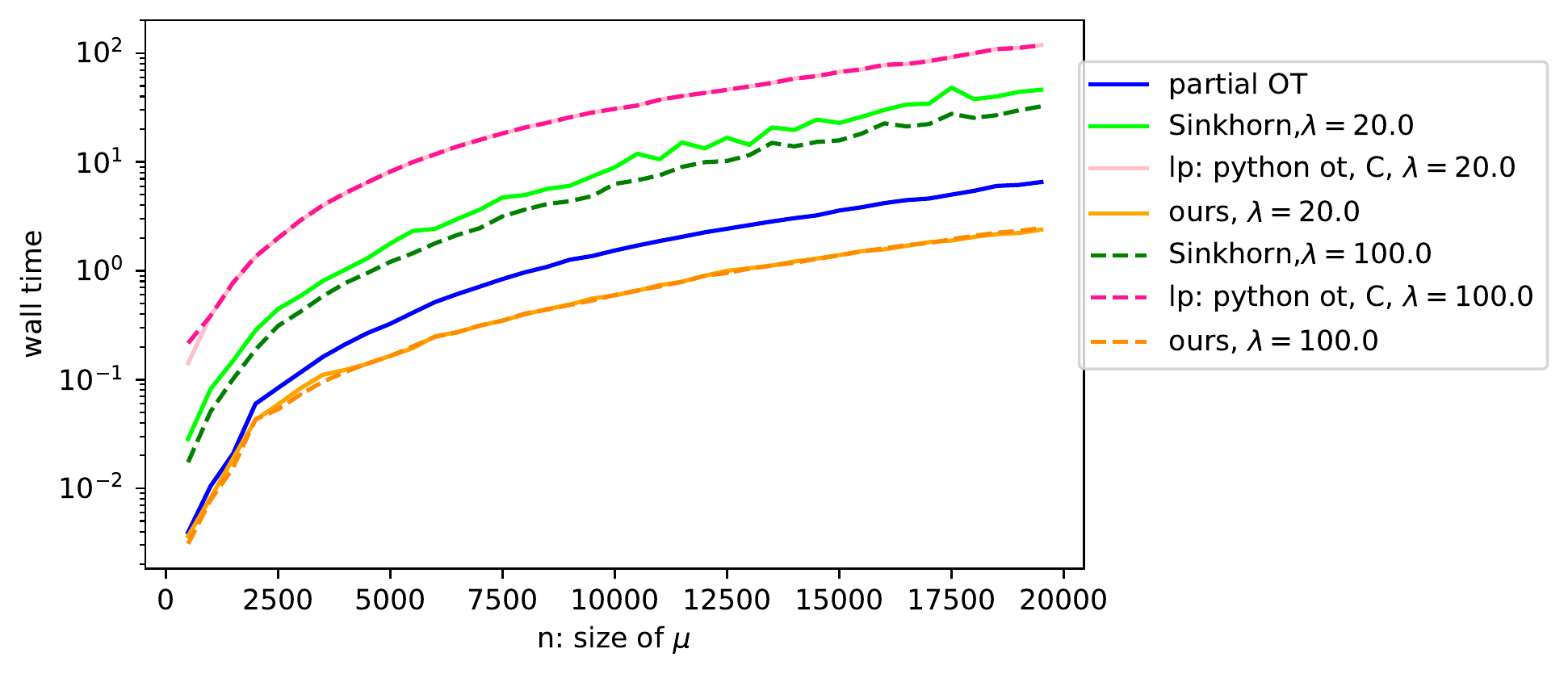}
    \caption{wall-clock time for uniform distributions}
\end{subfigure}
\begin{subfigure}{0.49\textwidth}
    \includegraphics[width=1\textwidth]{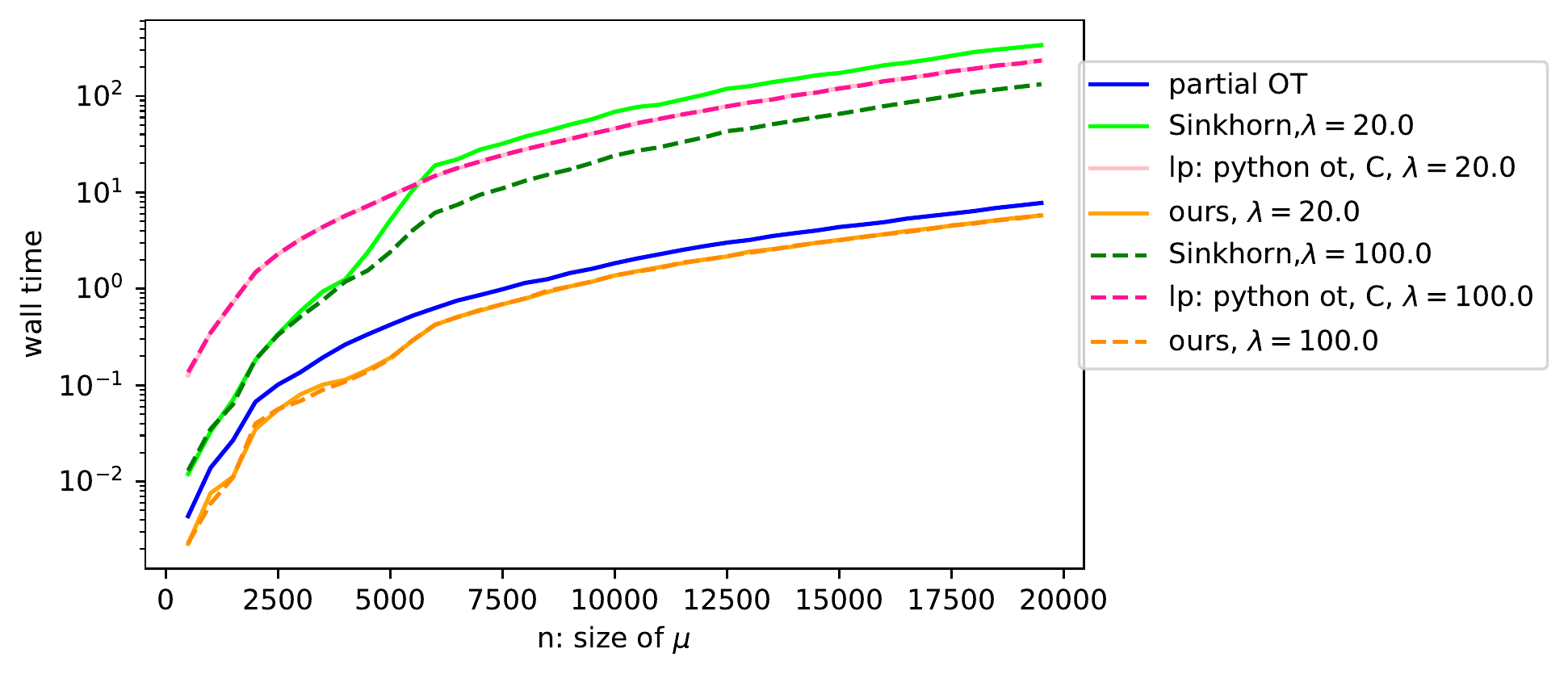}
    \caption{Wall-clock time for Gaussian mixture distribution}
\end{subfigure}
\hfill
\label{fig:wall time}
\caption{We test the wall-clock time of our Algorithm \ref{alg: 1d opt v1}, the partial OT solver (Algorithm 1 in \cite{Bonneel2019sliced}), the unbalanced Sinkhorn algorithm \cite{chizat2018scaling}, and the linear programming solver in POT \cite{flamary2021pot}, which is written in C++. The maximum number of iterations for linear programming and Sinkhorn is $200n\ln(n)$.}
\end{figure}

\begin{figure}[t!]
    \centering
    \includegraphics[width=.8\linewidth]{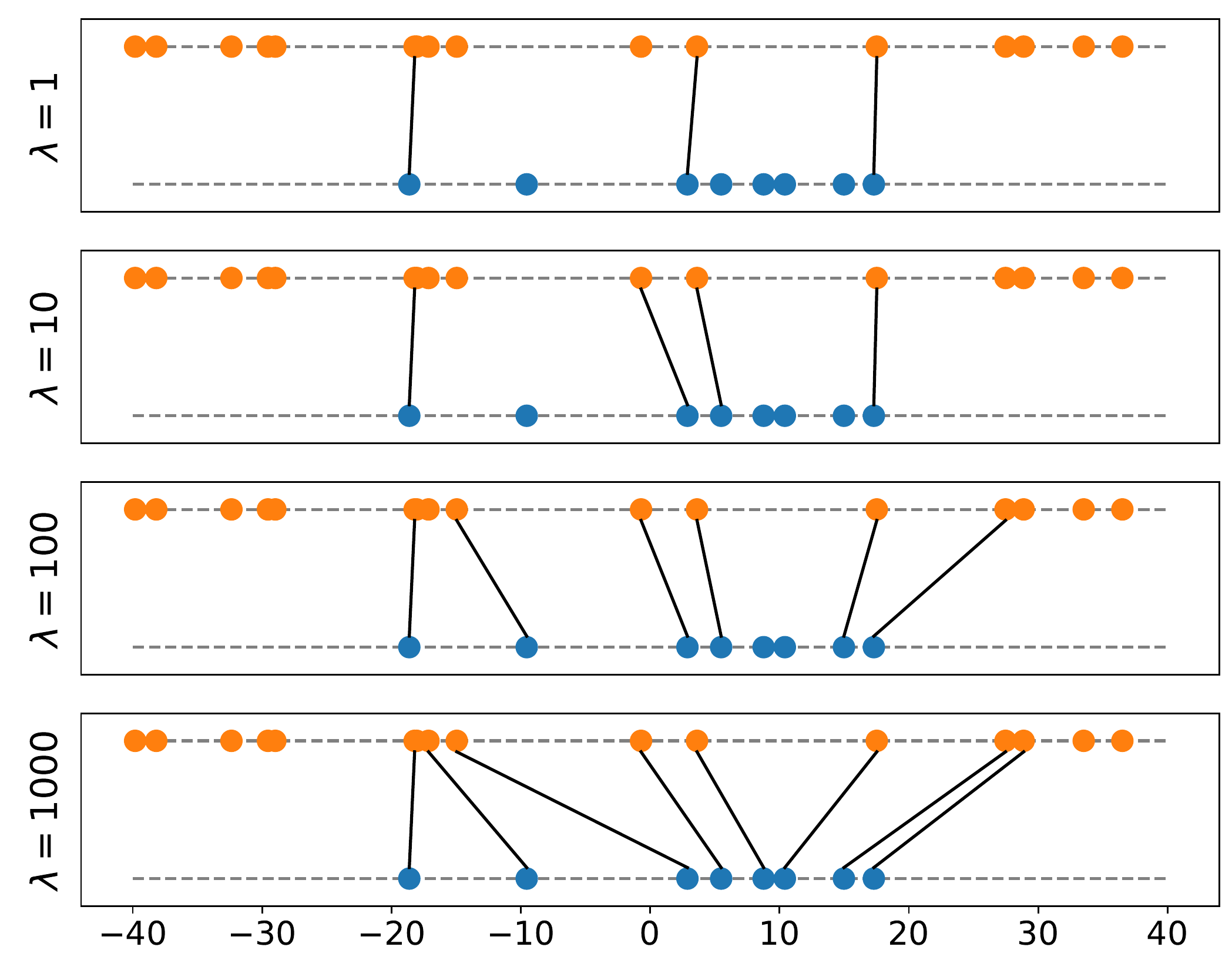}
    \caption{Output of the proposed algorithm for $\lambda\in[1,10,100,1000]$ on a sample pair of measures.}
    \label{fig: opt solution}
\end{figure}
\section{Sliced Optimal Partial Transport}
In practice, data has multiple dimensions and the 1-D computational methods cannot be applied directly. In the balanced OT setting the sliced OT approach \cite{rabin2011wasserstein,bonneel2015sliced,kolouri2015radon,kolouri2016sliced,kolouri2019generalized} applies the 1-D OT solver on projections (i.e., slices) of two r-dimensional distributions. Inspired by these works, we extend the sliced OT technique into the OPT setting and introduce the so-called ``sliced-unbalanced optimal transport'' problem.
\begin{definition}
In $\mathbb{R}^d$ space, given $\mu,\nu\in\mathcal{M}_+(\Omega)$ where $\Omega\subset\mathbb{R}^d$ and $\lambda:\mathbb{S}^{d-1}\to \mathbb{R}_{++}$ is an $\Lp{1}$ function, we define the sliced optimal partial transport (SOPT) problem as follows:
\begin{align}
\SOPT_\lambda(\mu,\nu)=\int_{\mathbb{S}^{d-1}} \OPT_{\lambda(\theta)}(\langle \theta,\cdot\rangle_\# \mu, \langle \theta,\cdot \rangle_\# \nu) \, \dd\sigma(\theta) \label{eq: sliced-opt}
\end{align}
where $\OPT_\lambda(\cdot,\cdot)$ is defined in \eqref{eq: OPT}, 
$\sigma\in\mathcal{P}(\mathbb{S}^{d-1})$ is a probability measure such that $\supp(\sigma)=\mathbb{S}^{d-1}$.   
\end{definition}
Generally $\sigma$ is set to be the uniform distribution on the unit ball $\mathbb{S}^{d-1}$. Whenever $\supp(\sigma)=\mathbb{S}^{d-1}$, $\SOPT_{\lambda}(\mu,\nu)$ defines a metric. 
\begin{theorem}\label{thm: sliced-opt}
Suppose $c:\mathbb{R}\times\mathbb{R}\to\bbR_+$ is the $p$-th power of a metric on $\mathbb{R}$, where $p\in[1,\infty)$, 
and $\lambda \in \Lp{1}(\bbS^{d-1};\bbR_{++})$, 
then 
$(\SOPT_{\lambda}(\mu,\nu))^{1/p}$ is a metric in $\mathcal{M}_+(\Omega)$.  
\end{theorem}
In practice, this integration is usually approximated using a Monte Carlo scheme that draws a finite number of i.i.d. samples $\{\theta_l\}_{l=1}^N$ from $\text{Unif}(\mathbb{S}^{d-1})$ and replaces the integral with an empirical average:
$$\SOPT_\lambda(\mu,\nu)\approx \frac{1}{N} \sum_{l=1}^N \OPT_{\lambda_l}(\langle \theta_l,\cdot\rangle_{\#}\mu,\langle\theta_l,\cdot\rangle_{\#}\nu).$$

\begin{proof}[Proof of Theorem \ref{thm: sliced-opt}]

 First we claim $\SOPT_{\lambda}(\cdot,\cdot):(\mathcal{M}_+(\Omega))^2\to \mathbb{R}_+$ is a well defined function. It is clear $\SOPT_{\lambda}(\cdot,\cdot)$ is a function with domain $\mathcal{M}_+(\Omega)^2$ and co-domain $\mathbb{R}\cup\{\pm\infty\}$.
Pick $\mu,\nu$, we will show $\SOPT_{\lambda}\in[0,\infty)$. 
We have
\begin{align}
\SOPT_\lambda (\mu,\nu)&=\int_{\mathbb{S}^{d-1}}\OPT_{\lambda(\theta)}(\langle\theta,\cdot\rangle_\#\mu,(\langle\theta,\cdot\rangle_\#\nu) \, \dd\sigma(\theta) \ge 0 \label{eq: sopt nonnegativity}
\end{align}
where the inequality follows from the fact $\OPT_{\lambda(\theta)}(\langle\theta,\cdot\rangle_\#\mu,(\langle\theta,\rangle_\#\nu)\ge 0,\forall \theta$.
It remains to show $\SOPT_{\lambda}(\mu,\nu)<\infty$. We have
\begin{align}
\SOPT_\lambda(\mu,\nu)
 &\leq \int_{\mathbb{S}^{d-1}} \frac{1}{2}\lambda(\theta)(\|\langle\theta,\cdot \rangle_\#\mu\|_{\TV}+\|\langle\theta,\cdot \rangle_\#\nu\|_{\TV}) \, \dd\sigma(\theta) \nonumber \\
 &=\frac{1}{2}(\mu(\Omega)+\nu(\Omega)) \int_{\mathbb{S}^{d-1}}\lambda(\theta) \, \dd \sigma(\theta)\nonumber\\
 &<\infty \nonumber 
\end{align}
where the first inequality follows by plugging zero measure into the cost function in \eqref{eq: OPT}; 
the second inequality holds since $\lambda$ is an $\Lp{1}$ function.

Next, we will show $\mu=\nu$ iff $\SOPT_\lambda(\mu,\nu)=0$. 
If $\mu=\nu$, we have for every $\theta$, $\langle \theta,\rangle_\#\mu=\langle \theta,\rangle_\#\nu$ and thus $\OPT_{\lambda(\theta)}(\langle \theta,\cdot\rangle_\# \mu,\langle \theta,\cdot\rangle_\# \nu)=0$. Therefore $\SOPT_{\lambda}(\mu,\nu)=0$. For the reverse direction, we suppose $\SOPT_\lambda (\mu,\nu)=0$. Since $\supp(\sigma)=\mathbb{S}^{d-1}$,
we have for every $\theta$, $\OPT_{\lambda(\theta)}(\langle \theta,\cdot\rangle_\# \mu,\langle \theta,\cdot\rangle_\# \nu)=0$. For every $\theta$, since $\lambda(\theta)>0$, and $\OPT_{\lambda(\theta)}(\cdot,\cdot)$ is a metric (see~\cite[Proposition 2.10]{chizat2018unbalanced} or~\cite[Proposition 5]{Piccoli2014Generalized}) when $p=1$ or the $p$-th power of a metric (see Theorem \ref{Thm: OPT is metric}), we have $\langle\theta,\cdot\rangle_\#\mu=\langle\theta,\cdot\rangle_\#\nu$. 
By the injectivity of Radon transform on measures (see~\cite[Proposition 7]{bonneel2015sliced}), we have $\mu=\nu$. 

For symmetry, we have 
\begin{align}
 \SOPT_\lambda(\mu,\nu)&=\int_{\mathbb{S}^{d-1}}\OPT_{\lambda(\theta)}(\langle\theta,\cdot\rangle_\# \mu,\langle\theta,\rangle_\# \nu) \, \dd \sigma(\theta)\nonumber \\ 
 &=\int_{\mathbb{S}^{d-1}}\OPT_{\lambda(\theta)}(\langle\theta,\cdot\rangle_\# \nu,\langle\theta,\cdot\rangle_\# \mu) \, \dd \sigma(\theta) \nonumber \\
 &=\SOPT_\lambda(\nu,\mu)
\end{align}
where the second equality follows from the fact $\OPT_{\lambda(\theta)}(\cdot,\cdot)$ is the $p$-th power of a metric for each $\theta$. 

For the triangle inequality, we choose $\mu_1,\mu_2,\mu_3$ and have:
{\small
\begin{align}
&\SOPT_\lambda(\mu_1,\mu_3)^{1/p}\nonumber\\
&=\left(\int_{\mathbb{S}^{d-1}}\OPT_{\lambda(\theta)}(\langle\theta,\cdot\rangle_\# \mu_1,\langle\theta,\cdot\rangle_\# \mu_3) \, \dd \sigma(\theta)\right)^{1/p} \nonumber \\ 
 &\leq \left\{\int_{\mathbb{S}^{d-1}}\left[(\OPT_{\lambda(\theta)}(\langle\theta,\cdot\rangle_\# \mu_1,\langle\theta,\cdot\rangle_\# \mu_2)^{1/p}+(\OPT_{\lambda(\theta)}(\langle\theta,\cdot\rangle_\# \mu_2,\langle\theta,\cdot\rangle_\# \mu_3))^{1/p}\right]^p\, \dd \sigma(\theta)\right\}^{1/p} \nonumber \\
 &\leq \left(\int_{\mathbb{S}^{d-1}}\OPT_{\lambda(\theta)}(\langle\theta,\cdot\rangle_\# \mu_1,\langle\theta,\cdot\rangle_\# \mu_2) \, \dd \sigma(\theta)\right)^{1/p}+\left(\int_{\mathbb{S}^{d-1}}\OPT_{\lambda(\theta)}(\langle\theta,\cdot\rangle_\# \mu_1,\langle\theta,\cdot\rangle_{\#}\mu_2) \, \dd \sigma(\theta)\right)^{1/p}
  \nonumber \\ 
&
=\SOPT_\lambda(\mu_1,\mu_2)^{1/p}+\SOPT_\lambda(\mu_2,\mu_3)^{1/p} \nonumber
\end{align}}

\noindent where the first inequality holds since $\OPT_{\lambda(\theta)}(\cdot,\cdot)$ is the p-th power of a metric for each $\theta$ by Theorem \ref{Thm: OPT is metric}; the second inequality follows from Minkowski inequality in $L_p(\mathbb{S}^{d-1})$.

\end{proof}

\section{Applications}
\subsection{Point Cloud Registration}
Point cloud registration is a transformation estimation problem
between two point clouds, which is a critical problem
in numerous computer vision applications \cite{szeliski2010computer,huang2021comprehensive}. In particular, given two 3-D point clouds,
 $x_i\sim\mu$, $i=1,\dots,n$ and $y_j\sim\nu$, $j=1,\dots,m$, one assumes there is an unknown mapping, $T$, that satisfies $\nu=T_\#\mu$. In many applications, the mapping $T$ is restricted to have the following form, $Tx=sRx+\beta$, where $R$ is a $3\times 3$ dimensional rotation matrix, $s>0$ is the scaling and $\beta$, called translation vector, is $3\times 1$ vector. The goal is then to estimate the transform, $T$, based on the two point clouds. 

\begin{table}
  \centering
  {
\begin{tabular}{ll||l||l||l}
\multicolumn{1}{c}{\bf}  
&\multicolumn{1}{c}{\bf 10k,5\%}
&\multicolumn{1}{c}{\bf 10k,7\%}
&\multicolumn{1}{c}{\bf 9k,5\%}
&\multicolumn{1}{c}{\bf 9k,7\%}
\\ \midrule
ICP-D &1.10(1.59) &3.60(0.11) & 1.65(1.13) &  2.04(2.28) \\
ICP-U  &3.72(0.53) &3.72(0.52) & 3.69(0.63) &  3.92(0.32) \\
SPOT &1.27(0.01) &1.40(0.15) & 1.42(0.13) &  1.53(1.50) \\
\textbf{Ours} &0.01(1e-3) &0.02(2e-3) & 0.20(0.09) & 0.33(0.03) \\
\end{tabular}
}
\caption{{\small We compute the mean (and variance, in parenthesis) of errors in the Frobenius norm between the ground truth and estimated transportation matrices for ICP(Du), ICP(Umeyama), SPOT and our method. We vary the size of the source point cloud from 9k to 10k samples, the percentage of noise from 5\% to 7\% (on both source and target datasets).}}
\label{tb: shape accuracy}
\end{table}
\begin{figure}
 \includegraphics[width=1.0\textwidth]{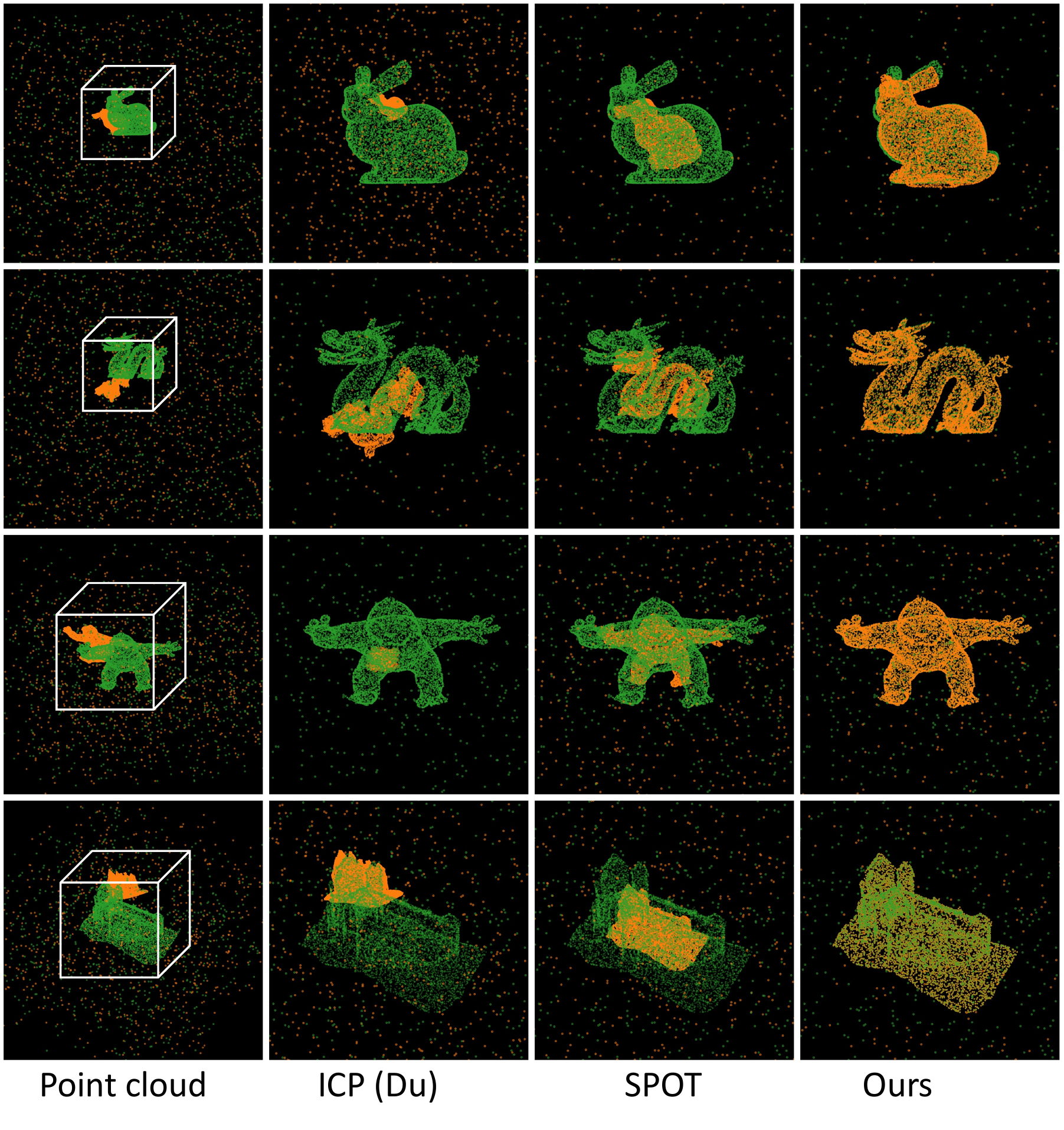}
\caption{We visualize the results of ICP(Du), SPOT and our method. In each image, the target point cloud is in green and source point cloud is in orange.}
\end{figure}

\begin{table*}[t]
{\small
\centering
\begin{tabular}{lll||ll||ll||ll}
\multicolumn{1}{c}{}  &\multicolumn{1}{c}{} 
&\multicolumn{1}{c}{}
&\multicolumn{2}{c}{\bf ICP}
&\multicolumn{2}{c}{\bf SPOT}
&\multicolumn{2}{c}{\bf  Ours}\\ \midrule
\multicolumn{1}{c}{\bf Dataset}  &\multicolumn{1}{c}{\bf Source} 
&\multicolumn{1}{c}{\bf Target}
&\multicolumn{1}{c}{\bf  Time/iter.}
&\multicolumn{1}{c}{\bf \# Iter.}
&\multicolumn{1}{c}{\bf Time/iter.}
&\multicolumn{1}{c}{\bf \# Iter.}
&\multicolumn{1}{c}{\bf Time/iter.}
&\multicolumn{1}{c}{\bf \# Iter.}
\\ \midrule
{Bunny (5\%)} & {$9k$} &{$10k$} & {$0.66s$} & {$50$} & {$10.9s$} & {$100$} & {$0.31s$} & {$2000$} \\
{Bunny (7\%)} & {$10k$} &{$10k$} & {$0.76s$} & {$130$} & {$13.2s$} & {$100$} & {$0.35s$} & {$2000$} \\ \midrule

{Dragon (5\%)} & {$9k$} &{$10k$} & {$0.66s$} & {$100$} & {$10.4s$} & {$100$} & {$0.31s$} & {$2000$} \\
{Dragon (7\%)} & {$10k$} &{$10k$} & {$0.77s$} & {$100$} & {$13.1s$} & {$80$} & {$0.35s$} & {$1500$} \\ \midrule

{Mumble (5\%)} & {$9k$} &{$10k$} & {$0.65s$} & {$100$} & {$10.78s$} & {$100$} & {$0.32s$} & {$2000$} \\
{Mumble (7\%)} & {$10k$} &{$10k$} & {$0.78s$} & {$100$} & {$13.18s$} & {$80$} & {$0.36s$} & {$1500$} \\ 
\midrule

{Castle (5\%)} & {$9k$} &{$10k$} & {$0.66s$} & {$150$} & {$10.7s$} & {$100$} & {$0.31s$} & {$2000$} \\
{Castle (7\%)} & {$10k$} &{$10k$} & {$0.76s$} & {$350$} & {$13.7s$} & {$80$} & {$0.35s$} & {$1800$} \\ \midrule
\end{tabular}}
\caption{This table reports the data for our method in shape registration experiment: the number of source and target distribution, the percentage of noise, the computation times per iteration and numbers of iterations they took to converge for ICP(Du), SPOT, and our method. The source point cloud is in orange color and the target is in blue. The percentage of noise is given in brackets next to the dataset name.}
\label{tb: shape registration time}
\end{table*}

The classic approach for solving this problem is Iterative Closest Point Algorithms (ICP) introduced by \cite{chen1992object,besl1992method}. By \cite{umeyama1991least}'s work, classical ICP can be extended into the uniformly scaled setting, with further developments by~\cite{du2007extension}. 
To address the some issues of ICP methods (convergence to local minimum, poor performance when the size of the two datasets are not equal), \cite{Bonneel2019sliced} proposed the Fast Iterative Sliced Transport algorithm (FIST) using sliced partial optimal transport (SPOT). 

\noindent\textbf{Problem setup and our method}.
We consider the uniform scaled point cloud registration problem and assume a subset of points in both the source and target point clouds are corrupted with additional uniformly distributed noise. 
We suppose prior knowledge of the proportion of noise is given (i.e., we have prior knowledge of the number of clean data). 

In general, the registration problem can be iteratively solved and each iteration contains two steps: estimating the correspondence and computing the optimal transform from corresponding points. The second step has a closed-form solution. For the first step, the ICP method estimates the correspondence by finding the closest $y$ for each transformed $x$. Inspired by this work, we estimate the correspondence by using our SOPT solver. See Algorithm \ref{alg: iterative-sopt}.  
\begin{algorithm}
\caption{iterative-sopt}\label{alg: iterative-sopt}
\KwInput{$\{x_i\}_{i=1}^{n},\{y_j\}_{j=1}^{m},$
$n_0$:the \# of clean $x$ , $N$: \# of projections
}
\KwOutput{$R,s,\beta$}
initialize $R,s,\beta,\lambda$, sample $\{\theta_i\}_{i=1}^N\subset\mathbb{S}^{2}$\\
\For{$l=1,\ldots N$}
{$\hat{Y}\gets sRX+\beta$\\ 
Compute transportation plan $L$ of $\OPT_{\lambda}(\theta_l^T\hat{Y},\theta_l^TY)$ by algorithm \ref{alg: 1d opt v1}\\ 
$\forall i\in \dom(L)$, $\hat{y}_i\gets\hat{y}_i+ (\theta_l^Ty_{L[i]}-\theta_l^T\hat{y}_i)\theta$\\
Compute $R,s,\beta$ from $(X[dom(L)],\hat{Y}[dom(L)])$ by ICP (e.g., equations (39)-(42) in \cite{umeyama1991least})\\ 
If $|dom(L)|>n_0$, decrease $\lambda$; otherwise, increase $\lambda$. \\
}
\end{algorithm}

\noindent\textbf{Experiment}.
We illustrate our algorithm on different 3D point clouds, including Stanford Bunny (\url{https://graphics.stanford.edu/data/3Dscanrep/}), Stanford dragon (\url{https://graphics.stanford.edu/data/3Dscanrep/}),  Witch-castle (\url{https://github.com/nbonneel/spot/tree/master/Datasets/Pointsets/3D}) and
Mumble Sitting (\url{https://github.com/nbonneel/spot/tree/master/Datasets/Pointsets/3D}).
For each dataset, we generate transforms by uniformly sampling angles from $[-1/3\pi,1/3\pi]$, translations from $[-2std,2std]$, and scalings from $[0,2]$, where $std$ is the standard deviation of the dataset. Then we sample noise uniformly from the region $[-M,M]^3$ where $M=\max_{i\in[1:n]}(\|x_i\|)$ and concatenate it to our point clouds. The number of points in the target (clean) data is fixed to be $10k$ and we vary the number of points in the source (clean) data from $9k$ to $10k$, and the percentage of noise from $5\%$ to $7\%$. 

\noindent\textbf{Performance}. For accuracy, we compute the average and standard deviation of error defined by the Frobenius norm between the estimated transformation and the ground truth (see Table \ref{tb: shape accuracy}). For accuracy,  we observe ICP methods systematically fail, since the nearest neighbor matching technique induces a non-injective correspondence, which may result in a too-small scaling factor. SPOT can successfully recover the rotation, but it fails to recover the scaling. Our method is the only one that recovers the ground truth for the noise-corrupted data (since it utilizes prior knowledge).

For running time, ICP has the fastest convergence time which is generally 100-260 seconds since finding the correspondence by the closest neighbor can be done immediately after the cost matrix is computed. SPOT requires 1000-1300 seconds and our method requires 500-700 seconds. The data type is $32$-bit float number and the experiment is conducted on a Linux computer with AMD EPYC 7702P CPU with 64 cores and 256GB DDR4 RAM.



\begin{figure}[t]
\centering
 \includegraphics[width=1\textwidth]{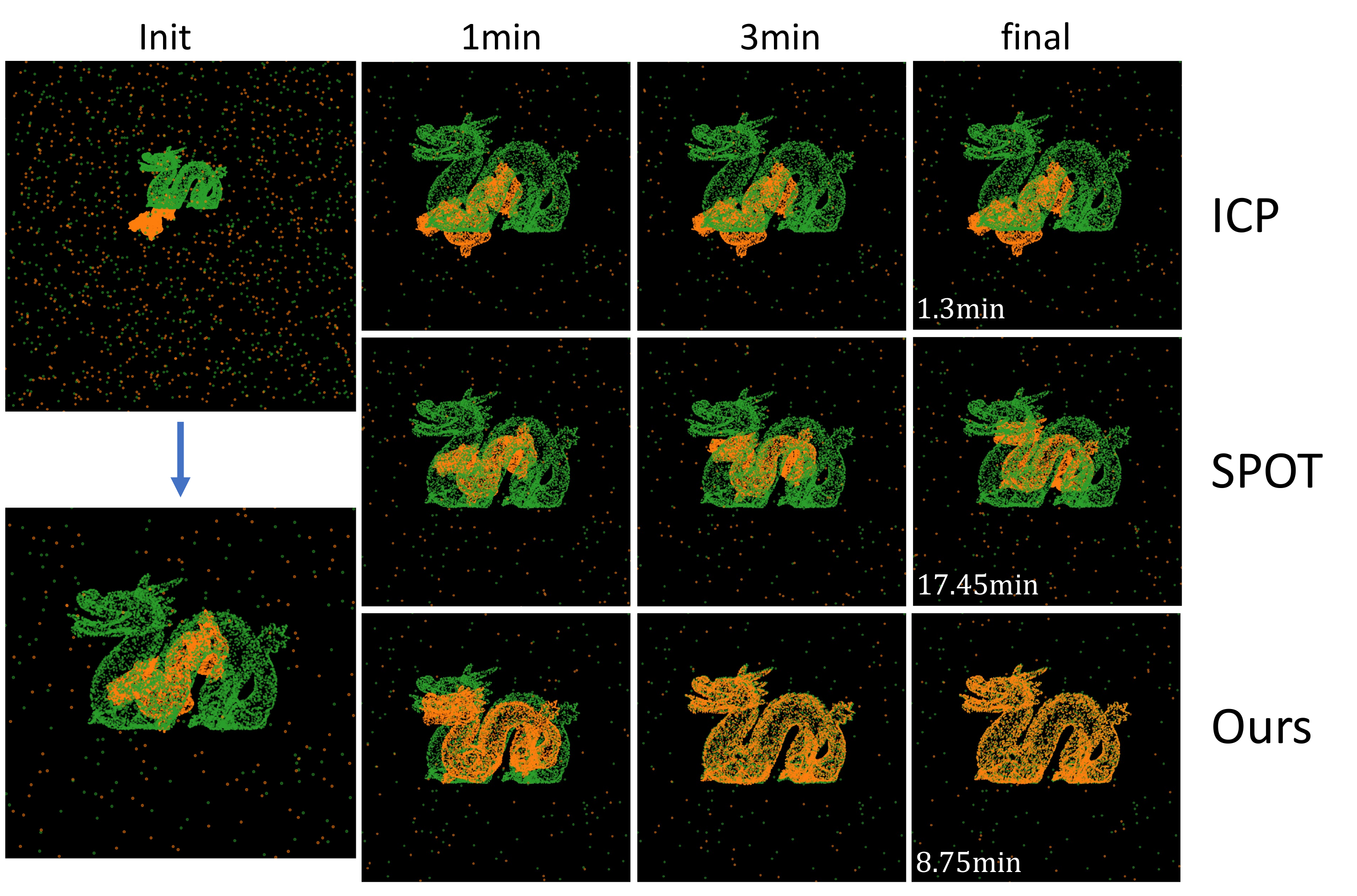}
\caption{We visualize the processed point cloud for every method with respect to time. The dataset is Stanford dragon
(\url{https://graphics.stanford.edu/data/3Dscanrep/}). In each image, the point cloud in orange is the source and the point cloud in green color is the target.}
\label{fig:data_results}
\end{figure}

\subsection{Color Adaptation}
Transferring colors between images is a classical task in computer vision and image science.
Given two images, the goal is to impose on one of the images (source image) the histogram of the other image (target image). Optimal transport-based approaches have been developed and achieved great success in this task \cite{chizat2018scaling,bonneel2015sliced,rabin2011wasserstein,ferradans2014regularized}. 
However, in the balanced OT setting, the OT-based approach requires normalizing the histograms of (sampled) colors, which may lead to undesired performance. For example, suppose the majority of a target image is red, (e.g. an image of evening sunset) and the majority of a source image is green (e.g. image of trees). Then balanced-OT-based approaches will produce a red tree in the result. To address this issue, \cite{Bonneel2019sliced} applied the SPOT-based approach which will match all the pixels in the source image to \textit{partial} pixels in the target image. 

\noindent\textbf{Our method.} 
Inspired by \cite{ferradans2014regularized,Bonneel2019sliced}, our method contains the following three steps: 
First, we sample pixels from the source and target image by k-mean clustering (or another sampling method). Second, we transport the sampled source pixels into the target domain. If OT or entropic OT is applied, it can be done by the optimal transportation plan; if sliced-OT is applied, the source pixels would be updated iteratively for each slice based on the gradient of 1-D OT with respect to the source pixels (see equation (46) in \cite{bonneel2015sliced}, or line 5 in our algorithm \ref{alg: iterative-sopt}). In our method, we apply the transportation plan from OPT.  
Third, we reconstruct the source image based on the transported source pixels (e.g. see Equation 4.1 in \cite{ferradans2014regularized}). 

\noindent\textbf{Experiment}. We first normalize all the pixels in the source and target images to be in the range $[0,1]$, then we use k-means clustering to sample $5000$ pixels from the source image and $10000$ pixels from the target image. 
We compare the performance of the OT-based and Entropic-OT-based domain adaptation functions in PythonOT \cite{flamary2021pot} (ot.da.EMDTransport and ot.da.SinkhornTransport) whose solver of OT is written in C++ \footnote{We modify their code to increase the speed.}, SPOT \cite{Bonneel2019sliced} and our method based on sliced optimal partial transport. For our method, we test it in two schemes, $\lambda=10.0$ and $\lambda<2.0$. In the first case, $\lambda$ achieves the maximum distance of two (normalized) pixels, that is, we will transport all the source pixels into the target domain. In the second case, we choose small $\lambda$, and theoretically, only partial source pixels will be transported into the target domain. 

\noindent\textbf{Performance}. In these examples, the OT-based approach which matches all (sampled) pixels of the source image to all pixels of a target image can lead to undesired results. For example, in the second row of Figure \ref{fig: color_results}, the third image has dark blue in the sky and red color on the ground. This issue is alleviated in SPOT and our method. In our method, when $\lambda=5.0$, we will transfer all the (sampled) pixels from source to target and the result is similar to the result of SPOT \footnote{We conjecture the two results are not exactly the same due to the randomness of projections.}. When $\lambda<2.0$, the resulting image is closer to the source image. 
The OT-based method requires 40-50 seconds (we set the maximum iteration number of linear programming to be 1000,000); the Partial OT method requires 80-90 seconds (the \# of projections is set to be 400) and our method requires 60-80 seconds (the \# of projections is set to be 400). The data type is 32-bit float number and the experiment is conducted on a Linux computer with AMD EPYC 7702P CPU with 64 cores and 256GB DDR4 RAM.

\begin{figure}[t]
\centering
\includegraphics[width=1.0 \linewidth]{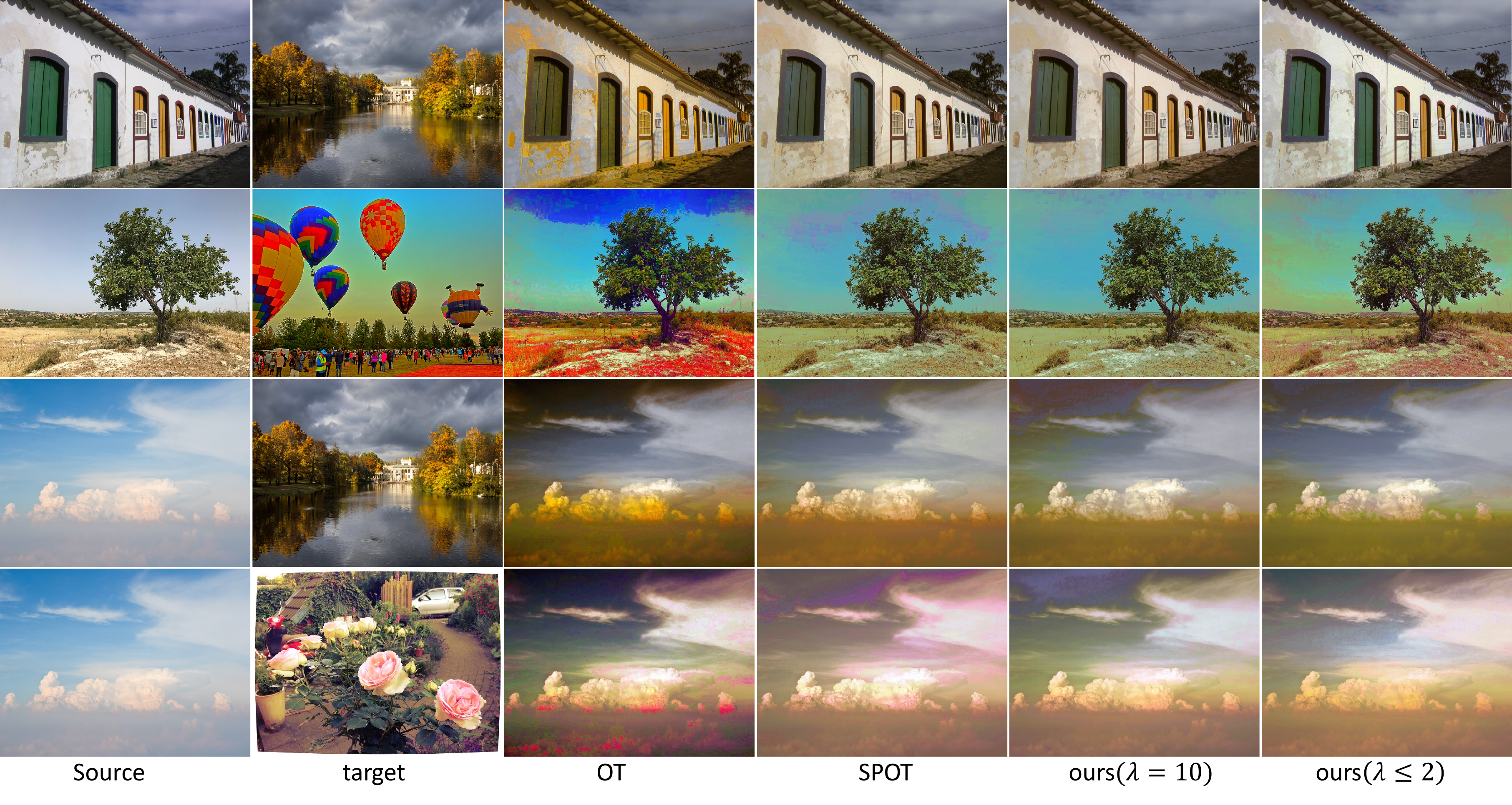}
\caption{We transfer colors from source image to the target image by the methods based on optimal transport \cite{ferradans2014regularized}, SPOT \cite{Bonneel2019sliced} and our SOPT. For our method, we set $\lambda=10$ and a small value less than $2$. 
Image via Flickr: Facade (\url{http://flic.kr/p/48kgPR}) by Phil Whitehouse, palace (\url{http://flic.kr/p/NJ6Vxq}) by Neil Williamson, clouds (\url{http://flic.kr/p/5mrPsc}) by Tim Wang, air balloon (\url{http://flic.kr/p/Ys81nY}) by Kirt Edblom, roses ( \url{http://flic.kr/p/f6bkoR}) by Felix Schaumburg.}
\label{fig: color_results}
\end{figure}


\section{Conclusion and future work}
This paper proposes a fast computational method for solving the OPT problem for one-dimensional discrete measures. We provide computational and wall-clock analysis experiments to assess our proposed algorithm's correctness and computational benefits. Then, utilizing one-dimensional slices of an $r$-dimensional measure we propose the ``sliced optimal partial transport (SOPT)'' distance. Beyond our theoretical and algorithmic contributions, we provide empirical evidence that SOPT is practical for large-scale applications like point cloud registration and image color adaptation.
In point cloud registration, we show that compared with other classical methods, our SOPT-based approach adds robustness when the source and target point clouds are corrupted by noise. In the future, 
we will investigate the potential applications of SOPT in other machine learning tasks 
such as open set domain adaptation problems and measuring task similarities in continual and curriculum learning.

\section{Acknowledgment}

The authors thank Rana Muhammad Shahroz Khan 
for helping in code testing and thank Dr. Hengrong Du (hengrong.du@vanderbilt.edu) for helpful discussions. This work was partially supported by the Defense Advanced Research Projects Agency (DARPA) under Contract No. HR00112190132.
MT was supported by the European Research Council under the European Union’s Horizon 2020 research and innovation programme Grant Agreement No. 777826 (NoMADS).
{\small
\bibliographystyle{documents/ieee_fullname}
\bibliography{documents/references}
}
\newpage
\appendix
\section{Relation Between Optimal Partial Transport and Unbalanced Optimal Transport}
If we replace the penalty term of \eqref{eq: OPT} with a constraint, i.e.~we impose the condition $\gamma(\Omega^2)\ge M$, and use the fact that $|\pi_{1\#}\gamma|=|\pi_{2\#}\gamma|=\gamma(\Omega^2)$, then \eqref{eq: OPT} is closely related to the Lagrangian formulation of the following ``primal problem'': 
\begin{align}
\text{Primal-OPT}(\mu,\nu;M)&=\inf_{\gamma\in\Pi_\leq(\mu,\nu)} \int c(x,y) \, \dd\gamma(x,y)  \quad\text{s.t. } \gamma(\Omega^2)\ge M. \label{eq: OPT primal}
\end{align}
This, in turn, is closely related to the optimal partial transport problem proposed by~\cite{caffarelli2010free,figalli2010new,figalli2010optimal} (the difference being the mass constraint of $\gamma$ is imposed as an equality $\gamma(\Omega^2)=M$ rather than a lower bound $\gamma(\Omega^2)\geq M$).

Another equivalent form of the OPT problem defined in \eqref{eq: OPT} is the ``generalized Wasserstein distance'' in
\cite{Piccoli2014Generalized,piccoli2016properties} 
(We refer to \cite[Proposition 1.1]{chizat2018interpolating} and \cite[Proposition 4]{Piccoli2014Generalized} for their equivalence.) Recently, more systematic studies of so-called ``unbalanced optimal transport'' or ``optimal entropy transportation'' problems have been conducted, for instance in \cite{chizat2018unbalanced} and \cite{Liero2018Optimal}. OPT, \eqref{eq: OPT}, can be seen as a special case of these models, see for instance \cite[Theorem 5.2]{chizat2018unbalanced}.
It is also well known that in addition to the static Kantorovich formulations presented here, one can also give equivalent dynamic formulations in the spirit of the Benamou--Brenier formula, e.g.~\cite{chizat2018unbalanced}.
Finally, a related class of models with close relations to the POT problem is discussed in \cite{lee2021generalized} under the name ``Generalized Unnormalized Optimal Transport''(GUOT).

\section{Relation Between Optimal Partial Transport and Optimal Transport} \label{sec: OPT and OT}
Inspired by Caffarelli et al.s' technique \cite{caffarelli2010free}, suppose $\Omega=\mathbb{R}^d$, we introduce an \textit{isolated point} $\hat{\infty}$ into $\Omega$ by letting $\hat{\Omega}=\Omega\cup \{\hat\infty\}$. 
Suppose $\hat{\mu}=\mu+(K-\mu(\Omega))\delta_{\hat{\infty}}$, where $\delta_{\hat{x}}$ is the Dirac mass at $\hat{x} \in \hat{\Omega}$ and $\hat{\nu}=\nu+(K-\nu(\Omega))\delta_{\hat{\infty}}$, where the constant $K$ satisfies $K\ge \mu(\Omega)+\nu(\Omega)$, and
$\hat{c}(x,y):\hat{\Omega}\times\hat{\Omega}\to\mathbb{R}_+$ is defined as 
$$\hat{c}(x,y):=\begin{cases}
c(x,y)-2\lambda &\text{if }x,y\neq \hat{\infty}\\
0 &\text{otherwise.}
\end{cases}$$
We introduce the following optimal transport problem:
\begin{align}
\inf_{\hat{\gamma}\in\Gamma(\hat{\mu},\hat{\nu})}\int\hat{c}(x,y)\, \dd\hat{\gamma}(x,y) \label{eq: opt-ot m}
\end{align}
We claim there exists an equivalence between this OT problem and $\OPT_\lambda(\mu,\nu)$. 
\begin{proposition}\label{Pro: OT and OPT}
The mapping: $T: \Gamma_\leq(\mu,\nu)\to \Gamma(\hat\mu,\hat\nu)$ defined by
\begin{equation}
\gamma\mapsto \hat{\gamma}=\gamma+(\mu-(\pi_1)_\# \gamma)\otimes \delta_{\hat\infty}+\delta_{\hat{\infty}}\otimes(\nu-(\pi_2)_\#\gamma)+(\gamma(\Omega^2)+\alpha)\delta_{\hat\infty,\hat\infty}, \label{eq: gamma and hat-gamma m}
\end{equation}
is a bijection, where $\alpha=K-(\mu(\Omega)+\nu(\Omega))$ and 
$\gamma$ is optimal in $\OPT_{\lambda}(\mu,\nu)$ if and only if $\hat{\gamma}$ is optimal in~\eqref{eq: opt-ot m}. 
\end{proposition}
\begin{proof}
First we will show that $\hat\gamma=T(\gamma) \in \Gamma(\hat\mu,\hat\nu)$ for $\gamma\in\Gamma_\leq(\mu,\nu)$.
Pick a Borel set $A\subset \hat\Omega$, and suppose $\hat{\infty}\in A$. 
By definition, $\hat{\gamma}$ is a measure defined on $\hat\Omega^2$, then we have
\begin{align}
\hat\gamma(A\times\hat{\Omega})&=\hat\gamma(A\setminus\{\hat\infty\}\times\Omega)+\hat{\gamma}(A\setminus\{\hat\infty\}\times\{\hat\infty\})+\hat\gamma(\{\hat{\infty}\}\times \Omega)+\hat{\gamma}(\{\hat\infty,\hat\infty\})\nonumber\\
&=\gamma(A\setminus\{\hat\infty\}\times \Omega)+(\mu-(\pi_1)_\#\gamma)(A\setminus\{\hat\infty\})+(\nu-(\pi_2)_\#\gamma)(\Omega)+\gamma(\Omega^2)+\alpha \nonumber \\
&=(\pi_1)_{\#}\gamma(A\setminus\{\hat\infty\})+(\mu-(\pi_1)_\#\gamma)(A\setminus\{\hat\infty\})+\nu(\Omega)-(\pi_2)_\#\gamma(\Omega)+\gamma(\Omega^2)+\alpha\nonumber \\ 
&=\mu(A\setminus\{\hat\infty\})+\nu(\Omega)+\alpha\nonumber \\
&=\hat\mu(A) \nonumber 
\end{align}
Similarly, if $\hat\infty\notin A$, we have $\hat\gamma(A\times \hat\Omega)=\mu(A)=\hat\mu(A)$. 
Thus $(\pi_1)_\# \hat\gamma=\hat{\mu}$ and similarly $(\pi_2)_\# \hat\gamma=\hat{\nu}$. 
Therefore $\hat{\gamma}\in \Gamma(\hat{\mu},\hat{\nu})$. 

It is obvious that the mapping $T$ is injective since if $\gamma_1\neq \gamma_2$ where $\gamma_1,\gamma_2\in \Gamma_\leq(\mu,\nu)$, then there exists one set $B\subset\Omega^2$ such that $\gamma_1(B)\neq \gamma_2(B)$. Then $\hat{\gamma}_1(B)=\gamma_1(B)\neq \gamma_2(B)=\hat{\gamma}_2(B)$. Therefore, $\hat\gamma_1\neq \hat\gamma_2$. 

Next, we will show the surjectivity of $T$. Pick any $\hat\gamma\in\Gamma(\hat\mu,\hat\nu)$, define $\gamma$ such that for any $B\subset \Omega^2$, $\gamma(B)=\hat\gamma(B)$. We have $\gamma\in\Gamma_\leq(\mu,\nu)$. Indeed, pick Borel set $A\subset\Omega$, we have $$\gamma(A\times\Omega)=\hat{\gamma}(A\times \Omega)\leq\hat{\gamma}(A\times\hat\Omega)=\hat{\mu}(A)=\mu(A).$$
Thus $(\pi_1)_\#\gamma\leq \mu$, similarly we have $(\pi_2)_\#\gamma\leq \nu$. 
Let 
$\hat{\gamma}_1=T(\gamma)$. We claim $\hat{\gamma}=\hat{\gamma}_1$.
Note, since $\Omega^2,\Omega\times\{\hat\infty\},\{\hat\infty\}\times\Omega,\{\hat\infty,\hat\infty\}$ is a disjoint decomposition of $\hat\Omega^2$ (and all of them are measurable), it is sufficient to prove 
$\hat\gamma(B)=\hat\gamma_1(B)$ for any Borel set $B$ which is a subset of one of these four sets. 

Case 1: If $B\subset\Omega^2$, we have $\hat\gamma_1(B)=\gamma(B)=\hat\gamma(B)$. 

Case 2: If $B=A\times \{\hat\infty\}$ where $A\subset\Omega$ is Borel set, then 
\begin{align}
 \hat\gamma(B)&=\hat\gamma (A\times \hat\Omega)-\hat{\gamma}(A\times \Omega)\nonumber \\
 &=\hat\mu(A)-\gamma(A\times \Omega)\nonumber\\
 &=\mu(A)-(\pi_1)_\# \gamma(A)\nonumber \\
 &=\hat\gamma_1(A\times\{\hat\infty\})=\hat\gamma_1(B) \nonumber 
\end{align}
Similarly, if $B=\{\hat{\infty}\}\times A$ for some $A\subset\Omega$, we have $\hat\gamma(B)=\hat\gamma_1(B)$. 

Case 3: If $B=\{(\hat\infty,\hat\infty)\}$. Note, since $\hat\gamma_1\in \Gamma(\hat{\mu},\tilde{\nu})$ as we discussed above, then $\hat\gamma(\hat\Omega^2)=\hat\gamma_1(\hat\Omega^2)$. Additionally, by Cases 1 and 2 we have $\hat\gamma(\Omega\times\Omega)=\hat\gamma_1(\Omega\times\Omega)$, $\hat\gamma(\Omega\times \{\hat\infty\})=\hat\gamma_1(\Omega\times \{\hat\infty\})$ and 
$\hat\gamma(\{\hat\infty\}\times\Omega)=\hat\gamma_1(\{\hat\infty\}\times\Omega)$. Thus
\begin{align}
\hat\gamma(B)&=\hat\gamma(\hat\Omega^2)-\hat\gamma(\Omega^2)-\hat\gamma(\Omega\times\{\hat\infty\})-\hat\gamma(\{\hat\infty\}\times \Omega)\nonumber\\
&=\hat\gamma_1(\hat\Omega^2)-\hat\gamma_1(\Omega^2)-\hat\gamma_1(\Omega\times\{\hat\infty\})-\hat\gamma_1(\{\hat\infty\}\times \Omega)\nonumber\\
&=\hat\gamma_1(B)
\end{align}
Hence, $\hat\gamma=\hat\gamma_1$ and thus that the mapping is surjective. 

We will show $\gamma$ is optimal in $\OPT_{\lambda}(\mu,\nu)$ if and only if $\hat{\gamma}$ is optimal in $\text{OT}(\hat\mu,\hat\nu)$ (defined in \eqref{eq: opt-ot m}). 
We let $C(\gamma),\hat{C}(\hat\gamma)$ denote the corresponding transportation cost of $\gamma, \hat{\gamma}$ with respect to the OPT, OT problems, i.e.
\begin{equation} \label{eq:app:Chat}
\hat{C}(\hat{\gamma}) = \int \hat{c}(x,y) \, \dd \gamma(x,y), \qquad C(\gamma) = \int c(x,y) \, \dd \gamma(x,y) + \lambda \left( \mu(\Omega) - \pi_{1\#}\gamma(\Omega) + \nu(\Omega)-\pi_{2\#}\gamma(\Omega) \right).
\end{equation}
We have $C(\gamma)=C(\hat\gamma)+\lambda(\mu(\Omega)+\nu(\Omega))$.  Combined with the fact the mapping is a bijection, we have $\gamma$ is optimal iff $\hat\gamma$ is optimal. 
\end{proof}

\section{OPT defines a metric}
When the cost function $c(x,y)$ is a p-th power of a metric, similar to OT, OPT can also define a metric in $\mathcal{M}_+(\Omega)$. For finite discrete $\mu,\nu$, a similar result has been proposed by \cite[Theorem 2.2]{heinemann2023kantorovich}. Here we propose the general version: 
\begin{theorem}[OPT defines a metric]\label{Thm: OPT is metric}
If $c(x,y):\Omega^2\to \mathbb{R}_+$ is defines as $c(x,y)=D^p(x,y)$ for some metric $D$ defined on $\Omega$ and $\lambda>0$, then $(\text{OPT}_\lambda(\cdot,\cdot))^{1/p}$ defines a metric in $\mathcal{M}_+(\Omega)$. 
\end{theorem}
\begin{proof}
It is straightforward to show $(OPT_\lambda(\cdot,\cdot))^{1/p}$ is symmetric and  $(OPT_\lambda(\mu,\nu))^{1/p}=0$ if and only if $\mu=\nu$. 
For the triangle inequality, let $\hat{\Omega},\hat\infty,\hat\mu,\hat\nu, K$ denote the corresponding concepts as defined in section \ref{sec: OPT and OT}.
By Lemma \ref{lem: truncated cost}, we can replace the cost function $D^p(x,y)$ by $D^p(x,y)\wedge 2\lambda$ in the OPT problem, and the optimal value remains unchanged. That is 
$$\text{OPT}_\lambda(\mu,\nu):=\inf_{\pi\in\Pi_\leq(\mu,\nu)}(D^p(x,y)\wedge 2\lambda) d\gamma+\lambda((\mu(\Omega)-\pi_{1\#}\gamma(\Omega))+(\nu(\Omega)-(\pi_{2\#}\gamma(\Omega)))$$
In addition, by proposition \ref{Pro: OT and OPT}, we have 
$$\gamma \mapsto \hat\gamma=\gamma+(\mu-(\pi_1)_\#\gamma)\otimes \delta_{\hat\infty}+\delta_{\hat\infty}\otimes(\nu-(\nu-\pi_2)_\#\gamma)+(\gamma(\Omega^2)+\alpha)\delta_{\hat\infty,\hat\infty}$$
is a bijection between $\Pi_\leq (\mu,\nu)$ and $\Pi(\hat\mu,\hat\nu)$, where $\alpha=K-\mu(\Omega)-\nu(\Omega)$. 
Let $C(\gamma;\mu,\nu,\lambda)$ denote the objective function of $OPT_\lambda(\mu,\nu)$. Follows the section 3.1 in \cite{heinemann2023kantorovich}, we define $D'(x,y):\Omega\cup\{\hat\infty\}\to \mathbb{R}_+$ such that 
$$(D')^p(x,y)=\begin{cases}
D^p (x, y)\wedge 2\lambda &\text{if }(x,y)\in\Omega \\ 
\lambda &\text{if }x\in\Omega, y=\hat\infty \text{ or vise verse}\\
0 &\text{if }x=y=\hat\infty
\end{cases}
$$
and $D'$ defines a metric. Furthermore, We define the following OT problem 
$$\text{OT}(\hat\mu,\hat\nu)=\inf_{\gamma\in\Gamma(\hat\mu,\hat\nu)}\int (D')^p(x,y)d\gamma(x,y)$$
and let $C(\hat\gamma;\hat\mu,\hat\nu)$ to be the corresponding objective function, i.e. 
$$C(\hat\gamma;\hat\mu,\hat\nu):=\int_{\hat\Omega}(D')^p(x,y)d\hat\gamma(x,y).$$

For each $\gamma\in \Pi_\leq (\mu,\nu)$, we have 
\begin{align}
&C(\gamma;\mu,\nu,\lambda)\nonumber\\
&=\int_{\Omega^2} (D^p(x,y)\wedge 2\lambda) d\gamma(x,y)+\lambda((\mu-\pi_{\#1}\gamma)(\Omega)+(\nu-\pi_{\#2}\gamma)(\Omega)) \nonumber\\
&=\int_{\Omega^2} (D')^p(x,y) d\gamma(x,y) +\int_{\Omega\times \{\hat\infty\}}\lambda d((\mu-\pi_{1\#}\gamma)\otimes \delta_{\hat\infty})+\int_{\{\hat\infty\}\times \Omega}\lambda d(\delta_{\hat\infty}\otimes (\nu-\pi_{2\#}\gamma))+\int_{\{(\hat\infty,\hat\infty)\}}0d\delta_{(\hat\infty ,\hat\infty)}\nonumber\\
&=\int_{\Omega^2} (D')^p(x,y)d\hat\gamma+\int_{\Omega\times\{\hat\infty\}}(D')^p(x,y)d\hat\gamma+\int_{\{\hat\infty\}\times \Omega}(D')^p(x,y)d\hat\gamma+\int_{\{(\hat\infty,\hat\infty)\}}(D')^p(x,y)d\hat\gamma \nonumber \\
&=\int_{\hat\Omega^2}(D')^p(x,y)d\hat\gamma \nonumber \\
&=C(\hat\gamma;\hat\mu,\hat\nu).\nonumber  
\end{align}
Combining with the fact $\gamma\mapsto \hat\gamma$ is bijection, we have 
$\text{OPT}_\lambda(\mu,\nu)=\text{OT}(\hat\mu,\hat\nu)$.

Choose $\mu_1,\mu_2,\mu_3\in\mathcal{M}_+(\Omega)$ and let $K=\mu_1(\Omega)+\mu_2(\Omega)+\mu_3(\Omega)$. Define $\hat\mu_1, \hat\mu_2, \hat\mu_3$ introduced in the section B. We have $\hat\mu_1(\Omega)=\hat\mu_2(\Omega)=\hat\mu_3(\Omega)=K$. Then since $OT(\cdot,\cdot)^{1/p}$ defines a metric, we have 
$$(OT(\hat\mu_1,\hat\mu_3))^{1/p}\leq  (OT(\hat\mu_1,\hat\mu_2))^{1/p}+(OT(\hat\mu_2,\hat\mu_3))^{1/p}.$$
Therefore, we have:
$$(OPT_\lambda(\mu_1,\mu_3))^{1/p}\leq  (OPT_\lambda(\mu_1,\mu_2))^{1/p}+(OPT_\lambda(\mu_2,\mu_3))^{1/p}.$$
\end{proof}

\section{Correctness and complexity of Algorithms~\ref{alg: 1d opt v1} and~\ref{alg: sub opt}}
\subsection{Correctness}
In this section we prove the correctness of Algorithms~\ref{alg: 1d opt v1} and~\ref{alg: sub opt} as stated above and we discuss how to deal with duplicate points. Extended versions of the Algorithms with more sophisticated data structures and proper handling of boundaries and duplicates are then given in Section \ref{sec:AlgoComplex} together with a bound on their worst-case complexity.

\paragraph{Preliminaries, induction strategy, cases 1 and 2.}
Throughout this proof we are simply going to write $c_{i,j}$ for $c(x_i,y_j)$.
We assume that the point lists $\{x_i\}_{i=1}^n$ and $\{y_j\}_{j=1}^m$ are sorted, but we now allow for duplicate points and their handling will be addressed throughout this proof. Since $c(x,y)=h(x-y)$ for $h$ strictly convex, it is easy to verify that
\begin{align}
	\label{eq:Monge}
	c_{i,j}+c_{k,l} \leq c_{i,l}  +c_{k,j}
\end{align}
if $i \leq k$ and $j \leq l$, with a strict inequality if $x_i < x_k$ and $y_j < y_l$.
This is known as Monge property \cite{ReviewMongeMatrix-96}.
The proof works via induction in the iterations of the main loop of Algorithm 1.
We will show that prior to the iteration for $x_k$ / after completing the iteration for $x_{k-1}$, the following holds:
\begin{enumerate}[I.]
\item $\Psi_j \leq \lambda$ for all $j \in [1:m]$. \label{item:PsiBound}
\label{item:First}
\item For all $j \in [1:m]$, if $\Psi_j < \lambda$, then $y_j$ is assigned.
 \label{item:PsiMarginal}
\item $\Phi_i \leq \lambda$ for all $i \in [1:n]$. \label{item:PhiBound}
\item For all $i \in [1:k-1]$, if $\Phi_i < \lambda$, then $x_i$ is currently assigned.
 \label{item:PhiMarginal}
\item All dual constraints $\Phi_i + \Psi_j \leq c_{i,j}$ for all $i \in [1:n]$, $j \in [1:m]$ hold. 
 \label{item:Constraints}
\item For all $i \in [1:n]$, $j \in [1:m]$, whenever $x_i$ is assigned to $y_j$, one has $\Phi_i + \Psi_j = c_{i,j}$. \label{item:Active}
 \item The assignment $L$ will be monotonous. I.e.~if $L[i]\neq -1$, $L[i'] \neq -1$ for $i<i'$, then $L[i]<L[i']$. \label{item:Monotone}
\label{item:Last}
\setcounter{listmemory}{\value{enumi}}
\end{enumerate}
We initialize with $\Psi_j=\lambda$ for all $j$, $\Phi_i=-\infty$ for all $i$, and empty assignment $L_i=-1$ for all $i$. Therefore, prior to the first iteration, when $k=0$, all conditions are satisfied.
Next, note that items \eqref{item:PsiBound} and \eqref{item:PsiMarginal} will be satisfied throughout the algorithm, since entries of $\Psi$ are only ever decreased during the algorithm; entries are only decreased when the corresponding $y_j$ are assigned; and once a point $y_j$ is assigned, it may become re-assigned, but it never becomes un-assigned again.
The claim that $y_j$ is never un-assigned is clear in all cases apart from Case 3.1. 
In Case 3.1 it follows from property~\eqref{item:Chain} below, which implies that when $x_{i_\Delta}$ is un-assigned from, say, $y_{j_\Delta}$, then $y_{j_\Delta}$ is re-assigned to $i_\Delta+1$ (since the assignment between $x_{i^\prime}$ and $y_{j^\prime}$ satisfies $L[i^\prime] = j_{\min}+i^\prime-i_{\min}$, i.e. the assignment is consecutive). 

\def\jlast{j_{\tn{last}}}%
Throughout the algorithm let $\jlast$ be the largest index among any assigned points $y_j$. We initially set $\jlast=-1$ when no $y_j$ is assigned.
Since assigned $y_j$ do not get un-assigned (merely re-assigned), $\jlast$ is non-decreasing.
\begin{lemma}
	\label{lem:jStarBiggerJLast}
	If $\jlast \neq -1$, then during the iteration of the main loop of Algorithm 1, for any minimizer $j^*$ in line 3 one has $y_{j^*} \geq y_{\jlast}$. In particular, $j^*$ can always be chosen such that $j^* \geq \jlast$.
\end{lemma}
\begin{proof}
If $\jlast=-1$ there is nothing to prove, since $j^* \geq 1$.
If $\jlast \neq -1$, then there must be some $i \in [1:k-1]$ such that $L[i]=\jlast$ and therefore $\Phi_i + \Psi_{\jlast}=c_{i,\jlast}$. After adjusting $\Phi_k$ in line 4 one has $\Phi_k + \Psi_{j^*}=c_{k,j^*}$. By dual feasibility we have in addition
\begin{align*}
	\Phi_k + \Psi_{\jlast} & \leq c_{k,\jlast}, &
	\Phi_i + \Psi_{j^*} & \leq c_{i,j^*}.
\end{align*}
Combining these four (in-)equalities we get
\begin{align*}
	c_{k,j^*}+c_{i,\jlast} \leq c_{k,\jlast} + c_{i,j^*}.
\end{align*}
If $x_k>x_i$, then by \eqref{eq:Monge} we have $y_{j^*} \geq y_{\jlast}$. So $j^* < \jlast$ can only happen if $y_{j^*} = y_{\jlast}$ and thus we may also choose $\jlast$ as minimizing index. Therefore, we may impose the constraint $j^* \geq \jlast$ in line 3.
In the case $x_k=x_i$, assume $\jlast$ would not be a minimal index in line 3, i.e.
$$c_{k,j^*} - \Psi_{j^*} < c_{k,\jlast} - \Psi_{\jlast}=\Phi_i,$$
where in the last equality we used $x_k=x_i$.
Since $L[i]=\jlast$ one must have that the dual constraint for $(i,\jlast)$ must be active. This would imply that the dual constraint for $(i,j^*)$ is violated, which contradicts the induction hypothesis.
\end{proof}

Now during iteration $k$, the change of $\Phi_k$ in line 4, by construction, preserves \eqref{item:PhiBound} and \eqref{item:Constraints}. Assume we enter Case 1. The assignment function $L$ is not changed, hence \eqref{item:Active} and \eqref{item:Monotone} remain preserved, and since $\Phi_k=\lambda$, \eqref{item:PhiMarginal} is extended to $i=k$. 
Assume we enter Case 2. Then we have $L_k=j^*$, $\Phi_k+\Psi_{j^*}=c_{k,j^*}$ and $\Phi_k < \lambda$. Hence, \eqref{item:Active} remains true, and \eqref{item:PhiMarginal} is extended to $i=k$. If we choose $j^* > \jlast$ (which is possible by Lemma \ref{lem:jStarBiggerJLast}), we preserve \eqref{item:Monotone}.

\paragraph{Case 3.}
We now turn to Case 3.

\begin{lemma}
\label{lem:LowerChainEndFirst}
	In each iteration of the main loop of Algorithm 1, when we enter Case 3, i.e.~$\Phi_k < \lambda$ and $j^*=\jlast$, let $i$ be the index such that $x_i$ is currently assigned to $y_{\jlast}$. Then $x_{i}=x_{i'}$ for all $i' \in [i:k-1]$. If $i<k-1$, then one must have $\Phi_i=\lambda$.
\end{lemma}
\begin{proof}
Clearly, $i < k$ (since it was assigned during a previous iteration).
In the following, let $f(x)=c(x,y_{\jlast})-\Psi_{\jlast}$, which is convex in $x \in \R$.	
By \eqref{item:Active} we have $\Phi_i=c(x_i,y_{\jlast})-\Psi_{\jlast}=f(x_i)\leq \lambda$, by the current iteration of the main loop we have $\Phi_k + \Psi_{\jlast}=c(x_k,y_{\jlast})=f(x_k)<\lambda$.
Let now $i' \in [i+1,k-1]$. By monotonicity of $L$, \eqref{item:Monotone}, if $L[i'] \neq -1$, then we would need $L[i']>\jlast$, which contradicts the definition of $\jlast$. Therefore $L[i']=-1$ and therefore by \eqref{item:PhiMarginal} we must have $\Phi_{i'}=\lambda$.
By \eqref{item:Constraints} we must also have $\Phi_{i'} \leq f(x_{i'})$, and by convexity of $f$, $f(x_{i'})$, since $f(x_i)\leq \lambda$, $f(x_k)<\lambda$, this can only happen if $x_{i'}=x_i$, and $\Phi_i=f(x_i)=\lambda$.
\end{proof}
This means that if all points are distinct, then we must have $i=k-1$, and find ourselves in the main loop of Algorithm 2, see below.
\begin{remark}
\label{rem:NonDistinct1}
If points are not necessarily distinct (at least up to numerical rounding errors), and we find $i<k-1$, then the situation can be remedied by setting $L[i]=-1$, $L[k]=j^*$, which preserves \eqref{item:Active}, \eqref{item:PhiMarginal} and \eqref{item:Monotone}. We will add this to the algorithms in Section \ref{sec:AlgoComplex}.
\end{remark}
We now study Algorithm 2 to resolve the conflict. In addition to the items above, at the beginning of each iteration of the main loop of Algorithm 2 the following is preserved:
\begin{enumerate}[I.]
\setcounter{enumi}{\value{listmemory}}
\item There are indices $j_{\min} \leq j^*$, $i_{\min} \leq k-1$ with $j^*-j_{\min}=(k-1)-i_{\min}$ such that $L[i_{\min}+r]=j_{\min}+r$ for $r \in [0:(k-1)-i_{\min}]$, $\Phi_{i+r}+\Psi_{j_{\min}+d-1}=c_{i+r,j_{\min}+d-1}$ for $r \in [1:(k-1)-i_{\min}]$.
\label{item:Chain}
\end{enumerate}
This is clearly true before the first iteration, when $i_{\min}=k-1$ and $j_{\min}=j^*$.
In each iteration of the main loop we then seek the largest possible value $\Delta \geq 0$ such that by setting $\Phi_i \gets \Phi_i + \Delta$ for $i \in [i_{\min},k]$ and $\Psi_j \gets \Psi_j - \Delta$ for $j \in [j_{\min},j^*]$ we preserve all items \eqref{item:First} to \eqref{item:Last}. Clearly the delicate ones are \eqref{item:PhiBound} and \eqref{item:Constraints}.
To preserve the former, we ensure that $\Delta \leq \lambda_{\Delta}$. To preserve the latter, we do not need to worry about the $\Psi_j$, $j \in [j_{\min}:j^*]$, since they are decreased, but we need to consider all constraints $\Phi_i+\Psi_j \leq c_{i,j}$ for $i \in [i_{\min}:k]$, $j \in [1:j_{\min}-1] \cup [j^*+1:m]$. By the following lemma, we see that this can be reduced to checking the two constraints for $(i_{\min},j_{\min}-1)$ and $(k,j^*+1)$, which is the role of the variables $\alpha$ and $\beta$ in Algorithm 2.

\begin{lemma}
    \label{lem:PathReduction}
	In the above situation, one has that
	\begin{align*}
	\min_{\substack{i \in [i_{\min}:k],\\ j \in [j^*+1:m]}} c_{i,j}-\Phi_i-\Psi_j & =c_{k,j^*+1}-\Phi_k-\Psi_{j^*+1}, \\
	\min_{\substack{i \in [i_{\min}:k],\\ j \in [1:j_{\min}-1]}} c_{i,j}-\Phi_i-\Psi_j & =c_{i_{\min},j_{\min}-1}-\Phi_{i_{\min}}-\Psi_{j_{\min}-1},
	\end{align*}		
	if $j^* <m$ and $j_{\min}>1$ respectively.
\end{lemma}
\begin{proof}
	We start with the first equation and begin by showing that
	\begin{equation}	
	\label{eq:ProofRowReduction}
	c_{i,j}-\Phi_i-\Psi_j \geq c_{k,j}-\Phi_k-\Psi_j
	\end{equation}
	for $i \in [i_{\min}:k]$, $j \in [j^*+1:m]$. We get this by combining $\Phi_i \leq c_{i,j^*}-\Psi_{j^*}$, $\Psi_{j^*}=c_{k,j^*}-\Phi_k$ and the Monge property of cost matrix, \eqref{eq:Monge}, $c_{k,j} \leq c_{i,j}+c_{k,j^*}-c_{i,j^*}$.
	Next, observe that $\Psi_j=\lambda$ for $j \in [j^*+1:m]$, since these values have not yet been changed since the initialization, and $\Psi_{j^*} \leq \lambda$.
	Also we know that $\Phi_k = c_{k,j^*}-\psi_{j^*} \leq c_{k,j}-\psi_{j}$ for all $j \in [1:m]$.
	Combining this, we get $c_{k,j^*} \leq c_{k,j}$ for $j \in [j^*+1:m]$. Since $f:y \mapsto c(x_k,y)$ is convex, and $f(y_{j^*}) \leq f(y_j)$ for $j \geq j^*$, we must have that $f$ is non-decreasing after $y_{j^*}$, and therefore among all indices $j \geq j^*$, the smallest one attains the minimum.
	
	Now we turn to the second equation. In complete analogy to \eqref{eq:ProofRowReduction} we show that $c_{i,j}-\Phi_i-\Psi_j \geq c_{i_{\min},j}-\Phi_{i_{\min}}-\Psi_j$ for $i \in [i_{\min}:k]$, $j \in [1:j_{\min}-1]$.
	Arguing as in Lemma \ref{lem:jStarBiggerJLast}, we can show a minimizing index $j$ can be chosen such that it is not smaller than $\hat{j}$, where $\hat{j}$ is the largest index among the assigned $y_j$, that is less than $j_{\min}$ (if such an assigned point exists, otherwise just let $\hat{j}=0$ in the following). Consequently, all $y_j$ points in $[\hat{j}+1:j_{\min}-1]$ must be unassigned and therefore have $\Psi_j=\lambda$. Arguing then via the convexity of $c$ as in the previous paragraph, we can show that a minimizing $j$ must be given by $j_{\min}-1$.
\end{proof}

The selection of Cases 3.1, 3.2 or 3.3 depends now on which of the three bounds $\lambda_\Delta$, $\alpha$, or $\beta$ is smallest.
Consequently, each of the implied updates of the dual variables in the three cases preserves the dual constraints and it is easy to see that by property \eqref{item:Chain} each of the conflict resolutions in Cases 3.1, 3.2 and 3.3a preserve all other conditions \eqref{item:First} to \eqref{item:Last}. For instance, when $\lambda_\Delta$ is minimal, element $x_{i_{\Delta}}$ becomes unassigned, however we then have $\Phi_{i_{\Delta}}=\lambda$ as required by \eqref{item:PhiMarginal}.

We are left with discussing Case 3.3b, i.e.~when $\beta$ is minimal among the three bounds and $y_{j_{\min}-1}$ is already assigned. In complete analogy to Lemma \ref{lem:LowerChainEndFirst} we can prove the following.
\begin{lemma}
\label{lem:LowerChainEnd}
	In each iteration of the main loop of Algorithm 2, when we enter Case 3.3b, i.e.~$\Phi_{i_{\min}} < \lambda$ and $y_{j_{\min}-1}$ is already assigned, let $i$ be the index such that $x_i$ is currently assigned to $y_{j_{\min}-1}$. Then $x_{i}=x_{i'}$ for all $i' \in [i:i_{\min}-1]$. If $i<i_{\min}-1$, then one must have $\Phi_i=\lambda$.
\end{lemma}
As above, this means that if all points are distinct, then $i=i_{\min}-1$, we can set $i_{\min} \gets i_{\min}-1$, $j_{\min} \gets j_{\min}-1$, note that we satisfy $\Phi_{i_{\min}}+\Psi_{j_{\min}}=c_{i_{\min},j_{\min}}$ and thus preserve \eqref{item:Chain} before the next iteration in Algorithm 2.
\begin{remark}
\label{rem:NonDistinct2}
If points are not all distinct and if $i<i_{\min}-1$, then we must have $\Phi_i=\lambda$, thus we can unassign $x_{i}$ and $y_{j_{\min}-1}$, and then proceed as if $y_{j_{\min}-1}$ were unassigned and resolve the conflict as in Case 3.3a.
\end{remark}

\subsection{Full algorithm versions and complexity}
\label{sec:AlgoComplex}
We now give more complete pseudo code versions of the Algorithms \ref{alg: 1d opt v1} and \ref{alg: sub opt}, see Algorithms \ref{alg: 1d opt v1 full} and \ref{alg: sub opt full}. The main purpose is to reach a quadratic worst case time complexity. Our algorithm can be seen as a specialization of the Hungarian method, exploiting the particular one-dimensional structure of the cost and dealing consistently with the option to discard mass for a cost $\lambda$.
The changes are explained below, subsequently some additional changes (for duplicate and boundary handling) are described in plain text, and finally we show how to determine the time complexity bound.
\def\OPT{\tn{OPT}}
\begin{algorithm}\caption{opt-1d}
\label{alg: 1d opt v1 full}
{\small
\KwInput{$\{x_i\}_{i=1}^n,\{y_j\}_{j=1}^m,\lambda$}
\KwOutput{$L$, $\Psi$, $\Phi$}
Initialize $\Phi_i\gets-\infty$ for $i\in [1:n]$, $\Psi_j\gets\lambda$ for $j\in [1:m]$ and $L_{i} \gets -1$ for $i\in [1:n]$,\\
$\jlast \gets 1$\\
\For{$k=1,2,\ldots n$}{
$j^*\gets\operatorname{argmin}_{j\in[\jlast:m]} c(x_k,y_j) - \Psi_j$\\
$\Phi_k\gets \min\{c(x_k,y_{j^*}) - \Psi_{j^*},\lambda\}$\\
\If{$\Phi_k=\lambda$}
{{\bf [Case 1]} No update on $L$}
{
\ElseIf{$j_{\min}-1$ unassigned}
{{\bf [Case 2]} $L_k \gets j^*$, $\jlast \gets j^*$
}
\Else 
{{\bf [Case 3]} Run Algorithm~\ref{alg: sub opt full}. }
}
}
}
\end{algorithm}
\begin{algorithm}
\caption{sub-opt-full}\label{alg: sub opt full}
{\small
\KwInput{($\{x_i\}_{i=1}^n,\{y_j\}_{j=1}^m$, $k$, $j^*$, $\jlast$, $L$, $\Phi$, $\Psi$)} 
\KwOutput{(Updated $L$, $\Phi$, $\Psi$, optimal for $\OPT(\{x_i\}_{i=1}^k,\{y_j\}_{j=1}^m)$, and $\jlast$)}
Initialize $i_{\min} \gets k-1$, $j_{\min}\gets j^*$.\\
Initialize $v \gets 0$, $d_j \gets 0$ for $j \in [1:m]$, $d_{k} \gets 0$, $d_{k-1} \gets 0$.\\
$i_{\Delta} \gets\operatorname{argmin}_{i\in[k-1:k]} (\lambda - \Phi_i)$, $\lambda_{\Delta} \gets \lambda - \Phi_{i_\Delta}$\\
\While{\text{true}}
{
$\alpha \gets c(x_k,y_{j^*+1})-\Phi_k-v-\Psi_{j^*+1}$ \label{line:alpha} \\
$\beta \gets c(x_{i_{\min}},y_{j_{\min}-1})-\Phi_{i_{\min}}-\Psi_{j_{\min}-1} \label{line:beta} $\\
\If{$\lambda_{\Delta}\leq \min\{\alpha,\beta\}$}
{[\textbf{Case 3.1}]\\
$v \gets v + \lambda_{\Delta}$\\
\For{$i \in [i_{\min},k-1]$}{$\Phi_i \gets \Phi_i + v - d_i$, $\Psi_{L_i} \gets \Psi_{L_i} -v+d_i$}
$\Phi_k \gets \Phi_k + v$\\
$L_{i_\Delta} \gets -1$, $L_{k} \gets j^*$ \\
\For{$i \in [i_\Delta+1:k-1]$}{$L_i \gets L_i-1$}
\textbf{return} \\
}
\ElseIf{$\alpha\leq \min\{\lambda_\Delta,\beta\}$}
{[\textbf{Case 3.2}]\\
$v \gets v + \alpha$\\
\For{$i \in [i_{\min},k-1]$}{$\Phi_i \gets \Phi_i + v - d_i$, $\Psi_{L_i} \gets \Psi_{L_i} -v+d_i$}
$\Phi_k \gets \Phi_k + v$\\
$L_k \gets j^*+1$, $\jlast \gets j^*+1$\\
\textbf{return}
}
\Else{
$v \gets v + \beta$\\
\If{$j_{\min}-1$ unassigned}{
[\textbf{Case 3.3a}]\\
\For{$i \in [i_{\min},k-1]$}{$\Phi_i \gets \Phi_i + v - d_i$, $\Psi_{L_i} \gets \Psi_{L_i} -v+d_i$}
$\Phi_k \gets \Phi_k + v$\\
$L_{i_{\min}} \gets j_{\min}-1$, $L_{k} \gets j^*$ \\ 
\For{$i \in [i_{\min}+1:k-1]$}{$L_i \gets L_i-1$}
\textbf{return}\\
}
\Else{
[\textbf{Case 3.3b}]\\
$d_{i_{\min}-1} \gets v$, $\lambda_\Delta \gets \lambda_\Delta - \beta$,\\
$i_{\min} \gets i_{\min}-1$, $j_{\min} \gets j_{\min}-1$\\
\If{$\lambda-\Phi_{i_{\min}} < \lambda_\Delta$}{
	$\lambda_\Delta \gets \lambda-\Phi_{i_{\min}}$, $i_{\Delta} \gets i_{\min}$}
}
}
}
}
\end{algorithm}
\paragraph{Implemented modifications compared to Algorithms \ref{alg: 1d opt v1} and \ref{alg: sub opt}.}
Compared to Algorithm  \ref{alg: 1d opt v1}, in Algorithm \ref{alg: 1d opt v1 full} we have added the variable $\jlast$ for improved handling of duplicate points (or limited numerical precision), see Remark \ref{rem:NonDistinct1}.
Note that initializing $\jlast \gets 1$, even when no points are yet assigned, yields the desired behavior. Additional adaptations related to this are discussed in the paragraph below.

Compared to Algorithm \ref{alg: sub opt}, there are several adaptations to Algorithm \ref{alg: sub opt full}.

The dual variables $\Phi$ and $\Psi$ are not updated during every loop of the algorithm but only once, when the conflict is resolved. This is handled via the auxiliary variable $v$ and the auxiliary array $d$. The former stores the total increment that will need to be applied to $\Phi_k$ at the end, in addition $d_i$ stores the value of $v$ at the time when $i$ was added to the `chain', therefore $v-d_i$ will be the necessary increment of $\Phi_i$ at the end.
This trick (which is also known for the Hungarian method) removes the necessity to loop over the whole chain to update the dual variables during each iteration of the main loop in Algorithm \ref{alg: sub opt full} and thus reduces the worst case time complexity from cubic to quadratic.

Similarly, the index of the dual variable $\Phi_i$ that is currently closest to $\lambda$ is not determined from scratch during each iteration. Instead, when case 3.3b is entered, the old best value is first reduced by $\beta$, then compared with the new competitor $i_{\min}$ (after updating $i_{\min}$), and updated if necessary.

\paragraph{Additional recommended modifications to algorithm.}
In lines \ref{line:alpha} and \ref{line:beta} of Algorithm \ref{alg: sub opt full} boundary checks should be added. E.g.~$\alpha$ can only be set as described if $j^* < m$, otherwise it should be set to $+\infty$. Likewise, $\beta$ can only be set as described if $j_{\min}>1$ and should otherwise be set to $+\infty$.
To keep track of which $y_j$ are assigned one can use a boolean array of size $m$, initialized with false, and entries corresponding to assigned points are set to true. This can be used to distinguish between cases 3.3a and 3.3b.
Alternatively, an `inverse' version of $L$ can be maintained, where $L^{-1}[j]$ will store the index $i$ of point $x_i$ to which point $y_j$ is assigned (and $-1$ otherwise). This has to be updated consistently with $L$. The latter will be useful when dealing with duplicate points according to Remarks \ref{rem:NonDistinct1} and \ref{rem:NonDistinct2} at the beginnings of case 3 and case 3.3b, respectively.

\paragraph{Worst case time complexity.}
In Algorithm \ref{alg: 1d opt v1 full}, initialization of the arrays $\Phi$, $\Psi$ and $L$ requires $\Theta(n+m)$ steps.
The main loop runs exactly $n$ times.
Determining $j^*$ requires $O(m)$ steps (using the particular structure of $c$ and $\Psi_j=\lambda$ for $j >j^*$ one could reduce this further, see Lemma \ref{lem:PathReduction}, but we leave such optimizations for future work).
Cases 1 and 2 take $\Theta(1)$ steps.
Let us now consider case 3 and Algorithm \ref{alg: sub opt full}. Initialization takes $\Theta(m)$ for setting up $d$ (although we note that this initialization could be skipped).
Cases 3.1, 3.2 and 3.3a are each entered only once, right before termination of the sub-routine, and they have a complexity of $O(n)$ (iterating over the chain for a fixed number of times to adjust $L$ and the duals).
Case 3.3b, as well as maintaining the variables $\alpha$, $\beta$ and $\lambda_\Delta$ have a complexity of $\Theta(1)$ per iteration and there are $O(n)$ iterations.
Hence, Algorithm \ref{alg: sub opt full} in its current form has a complexity of $O(\max\{m,n\})$, and therefore finally Algorithm \ref{alg: 1d opt v1 full} has a complexity of $O(n \cdot \max\{m,n\})$. During the proof we have pointed out some potential for optimizing the algorithm for the regime when $n \ll m$.


\end{document}